\documentclass[11 pt, onecolumn, draftcls]{IEEEtran}

\usepackage{amsbsy}
\usepackage{floatflt} 

\usepackage{amsmath}
\usepackage{amssymb}
\usepackage{times}
\usepackage{graphicx}
\usepackage{xspace}
\usepackage{paralist} 
\usepackage{setspace} 
\usepackage{xypic}
\xyoption{curve}
\usepackage{latexsym}
\usepackage{theorem}
\usepackage{ifthen}
\usepackage{subfigure}
\usepackage{turnstile}
\usepackage{mathtools}
\usepackage{booktabs}
\usepackage{balance}
\usepackage{url}

\newtheorem{theorem}{Theorem}[section]
\newtheorem{lemma}{Lemma}[section]
\newtheorem{proposition}{Proposition}[section]
\newtheorem{corollary}{Corollary}[section]

\newtheorem{definition}{Definition}[section]
\newtheorem{remark}{Remark}[section]

\def \Vap{\varepsilon}

\def \C{\mathcal{C}}
\def \PC{\mathcal{C}^{\perp}}
\def \OL{\overline{L}}

\def \mziu{\mathbf{z}_{i,u}}
\def \bniu{\bar{\mathbf{\nu}}_{i,u}}
\def \aiu{\alpha_{i,u}}
\def \biu{\beta_{i,u}}

\def \leqc{\leq_{c}}
\def \Qiu{\mathbf{Q}_{i,u}}
\def \Giu{\mathcal{G}_{i,u}}
\def \niu{\mathbf{\nu}_{i,u}}
\def \hg{\widehat{\gamma}}
\def \hvap{\widehat{\varepsilon}}
\def \bQ{\overline{\mathbf{Q}}}
\def \hL{\widehat{L}}
\def \hmu{\overline{\mathbf{U}}}
\def \hmj{\overline{\mathbf{J}}}
\def \oz{\overline{z}}
\def \wmz{\widehat{\mathbf{z}}}
\def \wmu{\widehat{\mathbf{U}}}
\def \wmj{\widehat{\mathbf{J}}}
\def \bG{\overline{\mathcal{G}}}
\def \bGiu{\overline{\mathcal{G}}_{i,u}}
\def \RXU{\mathbb{R}^{|\mathcal{X}\times\mathcal{U}|}}
\def \sziu{z_{i,u}}
\def \bviu{\overline{\Vap}_{i,u}}
\def \bQiu{\overline{Q}_{i,u}}
\def \hQiu{\widehat{Q}_{i,u}}
\def \pdelta{\delta^{\prime}}
\def \wziu{\widehat{z}_{i,u}}

\mathtoolsset{showonlyrefs}

\title{$QD$-Learning: A Collaborative Distributed Strategy for Multi-Agent Reinforcement Learning Through Consensus + Innovations}

\author{Soummya Kar$^{*}$, Jos\'e M.~F.~Moura$^{*}$ and H. Vincent Poor$^\dagger$
\thanks{$*$ Soummya Kar and Jos\'e M.~F.~Moura are with the Department of Electrical and Computer Engineering,
Carnegie Mellon University, Pittsburgh, PA, USA 15213 (e-mail:
soummyak@andrew.cmu.edu, moura@ece.cmu.edu).}
\thanks{$\dagger$ H. Vincent Poor is with the Department of Electrical Engineering, Princeton University, Princeton, NJ 08544, USA (e-mail: poor@princeton.edu).}
\thanks{The work of Jos\'e M.~F.~Moura was supported in part by the National Science Foundation under Grant \#~CCF--1011903 and in part by the Air Force Office of Scientific Research under Grant \#~FA--95501010291.

The work of H. Vincent Poor was supported by the National Science Foundation under Grant \#~DMS--1118605.}}

\begin{document}
\maketitle \thispagestyle{empty} \maketitle

\begin{abstract} The paper considers a class of multi-agent Markov decision processes (MDPs), in which the network agents respond differently (as manifested by the instantaneous one-stage random costs) to a global controlled state and the control actions of a remote controller. The paper investigates a distributed reinforcement learning setup with no prior information on the global state transition and local agent cost statistics. Specifically, with the agents' objective consisting of minimizing a network-averaged infinite horizon discounted cost, the paper proposes a distributed version of $Q$-learning, $\mathcal{QD}$-learning, in which the network agents collaborate by means of local processing and mutual information exchange over a sparse (possibly stochastic) communication network to achieve the network goal. Under the assumption that each agent is only aware of its local online cost data and the inter-agent communication network is \emph{weakly} connected, the proposed distributed scheme is almost surely (a.s.) shown to yield asymptotically the desired value function and the optimal stationary control policy at each network agent. The analytical techniques developed in the paper to address the mixed time-scale stochastic dynamics of the \emph{consensus + innovations} form, which arise as a result of the proposed interactive distributed scheme, are of independent interest.
\end{abstract}


\section{Introduction}
\label{introduction}

\subsection{Background and Motivation}
\label{backmot} This paper is motivated by problems of multi-agent decision-making in dynamic and uncertain environments. The basic setup consists of a network of \emph{agents} and a controlled global state process or signal (a finite state Markov chain with controlled transitions). The state process is actuated by a remote controller whose actions and the resulting controlled state influence the statistical distribution of the random instantaneous costs incurred at the agents. The Markov decision process (MDP) that we consider pertains to collaborative welfare, i.e., specifically, the agent network is interested in obtaining the optimal stationary control strategy that minimizes the network-averaged infinite horizon discounted cost. Our multi-agent setup, for instance, resembles that of a thermostatically controlled \emph{smart} building, in which the global state represents environmental dynamics affecting the spatial temperature distribution and the agents correspond to sensors distributed throughout the building. In this application, the objective of the building thermostatic controller is possibly of the reference tracking form, i.e., for example, to minimize the average of the squared deviations of the measured temperatures at the sensing locations from a desired reference value. It is important to note that the term agent has a generic usage here, whose scope varies from one application to the other. As another example, in which the agents correspond to social or organizational entities, consider a financial market setting. Here, the global signal may often be related to the dynamic market interest rate affecting, for example, the investment patterns of the agents, in which case the economic policies (actions) of the regulator (controller) may be shaped by the welfare motive to sustain an overall economic growth. The scope of our formulation is not limited to the above examples, and several practical scenarios, ranging from large-scale load control for efficient demand-side management in energy networks~\cite{CH} to collaborative decision-making in multi-agent robotic networks~\cite{Veloso-RoboSoccer,Yuta}, abound that motivate our setup.

Reinforcement learning, of which $Q$-learning~\cite{Watkins-Q, Tsitsiklis-Q, Jaakkola-Q} is an instance, has proved to be a valuable practically applicable solution methodology for MDPs in scenarios involving lack of prior information on the problem statistics, that includes the transition behavior of the controlled state process and, in our multi-agent setting, the statistical distributions of the agents' instantaneous costs (generally varying from one agent to the other). Based on a reformulation of the Bellman equation, the class of $Q$-learning methods generate sequential (stochastic) approximations of the value function using instantiations of state-action trajectories, as opposed to relying on exact problem statistics. The state-action trajectory instantiations for value function learning may correspond to online real-time data obtained while implementing the control, for example,~\cite{Tsitsiklis-Q}, in which case the resulting $Q$-learning methods are, in fact, instances of direct adaptive control~\cite{Sutton-Q}, or, may correspond to training data obtained through simulated state-action responses, see~\cite{Barto-Q} for various exploration methods. However, a direct application of the above classical reinforcement learning techniques to our proposed multi-agent setting with possibly geographically distributed agents would correspond to the requirement that there exists a centralized computing architecture having access to the instantaneous one-stage costs of all the agents at all times (see Section~\ref{sys-model} for a more detailed and formal discussion). Since the instantaneous one-stage costs may only be observed locally at the agents, this, in turn, requires each network agent to forward its one-stage cost to the remote central location at all times, which may not be feasible due to limited energy resources at the agents and a bit-budgeted communication medium. This motivates us to consider a fully distributed alternative, the $\mathcal{QD}$-learning, in which the agents participate in autonomous in-network learning by means of local computation and communication over a \emph{sparse} possibly time-varying communication network.

There has been extensive research on multi-agent reinforcement learning (see~\cite{Shoham-Survey,Busoniu-Survey} for surveys). Various formulations, ranging from general competitive dynamic stochastic games~\cite{Littman-SG,Littman-SG1,Hu-SG,Bowling-SG} to so called fully cooperative~\cite{Claus-SG, Busoniu-38, Busoniu-41, Busoniu-45, Busoniu-46, Busoniu-48, Busoniu-78, Busoniu-63}, have been investigated (see~\cite{Busoniu-Survey} for a more complete taxonomy). From the network objective viewpoint, the fully cooperative formulations are somewhat similar in spirit to our setup, in that both consider the optimization of a unique global quantity (the one-stage global cost corresponding to the network average of the random one-stage local agent costs in the current setting) -- the key difference being that in the current formulation we impose the additional constraint that the instantaneous random realizations of the one-stage global costs are not directly observable at the agents. More specifically, at a given time instant, each agent has access to its local instantaneous one-stage cost only and not their network average; whereas (often by problem definition), the fully cooperative formulations mentioned above (see also~\cite{Veloso-elsevier} for several decentralized variants) assume that the global one-stage costs are available at the agents at all times. Although, not directly comparable as the afore-mentioned approaches often involve decentralized actuation at the agent level as opposed to a remote process controller in our framework, we emphasize that, in the current context, they would require the network-average of the local instantaneous one-stage costs to be available at all agents at all times, which, given that the agents may be geographically distributed, would correspond to all-to-all agent communication at all times. What contrasts our proposed distributed approach from the existing literature is that we consider a fully distributed setting in which the agents disseminate the locally sensed costs through mutual neighborhood communication over a (pre-specified) sparse communication graph.

However, we also point out two aspects of generic multi-agent reinforcement learning that are not addressed by the current formulation, namely, that of partial state observation and decentralized actuation. Specifically, we assume that each network agent can perfectly access or observe the global state. Moreover, in contrast to setups with local decentralized agent actuations, our framework applies to models in which the control actions are generated by a remote (global) controller. Further, we assume that the remote control actions are perfectly known at the agents\footnote{Note that, for our setup involving a remote controller, the assumption that each network agent has access to the control actions may not be restrictive for applications of interest. For instance, often the control actions or decisions of the remote global policy maker (controller) may be directly observable to the network entities, for example, in financial or social network applications, the market or network entities are typically informed about the policies of the global welfare organization. In situations, where such direct observability is not possible, the control information may often be disseminated through network-wide broadcasts by the remote controller; given that the remote controller, being a global entity, has sufficient energy resources and that the action broadcasts are finite-bit (due to the finiteness of the action space), the action observability assumption may not be restrictive in many scenarios of interest.}, which could be a limitation in some applications.

Our distributed approach is of the \emph{consensus + innovations} type~\cite{KarMouraRamanan-Est}, in which the agents simultaneously incorporate the information received from their communicating neighbors and the instantaneous locally sensed costs in the same update rule (see also~\cite{JSTSP-Kar-Moura,SICON-Kar-Moura-Poor,Sayed-LMS,Bajovic-detection} and~\cite{dimakiskarmourarabbatscaglione-11} for related literature). As such, the resulting value function update processes at the agents are mixed time-scale, in which the distinct potentials of consensus (corresponding to information mixing through neighborhood communication~\cite{jadbabailinmorse03,SensNets:Olfati04,olfatisaberfaxmurray07,dimakiskarmourarabbatscaglione-11,karmoura-randomtopologynoise,Nedic,tsitsiklisbertsekasathans86}) and local innovation (corresponding to the instantaneous locally sensed one-stage cost) are traded off appropriately. Without inter-agent communication (the consensus potential), the locally sensed one-stage costs at the agents are not sufficient to provide an observable approximation of the desired global cost functional. On the other hand, given that the inter-agent communication is not all-to-all, exact reconstruction of the instantaneous global one-stage cost is not possible, and, hence, it is imperative to appropriately balance the two potentials so that in the long term the network information diffuses sufficiently to guarantee \emph{asymptotic} global cost observability at the agents. By suitably designing the time-varying weight sequences associated with the consensus and innovation potentials, we show that the $\mathcal{QD}$-learning achieves optimal learning performance asymptotically, i.e., the network agents reach consensus on the desired value function and the corresponding optimal stationary control strategy, under minimal connectivity assumptions on the underlying communication graph (see Section~\ref{alg_DQ} for details). Similar to direct adaptive control formulations (see, for example,~\cite{Tsitsiklis-Q}), we allow generic statistical dependence on the state-action trajectories (processes) that drive the learning, which, in turn, in our distributed setting, leads to mixed time-scale stochastic evolutions that are non-Markovian. The analysis methods developed in the paper are of independent interest and we expect our techniques to be applicable to broader classes of distributed information processing and control problems with memory. From a technical viewpoint, in centralized or single-agent operation scenarios, the connection between $Q$-learning and stochastic approximation was made explicit in~\cite{Tsitsiklis-Q}. In this paper, we develop a distributed generalization of $Q$-learning, $\mathcal{QD}$-learning, along the lines of consensus and innovations, thus extending the above connection to distributed multi-agent scenarios.

On another note, the work in this paper is also related to problems of distributed optimization in multi-agent networks. The existing literature on distributed optimization (see, for example,~\cite{jakoveticxaviermoura-11,Giannakis-opt,Nedic-opt,rabbatnowakbucklew05,Ram-Nedich-Siam}) mostly consider static scenarios, in which, broadly the network goal is to minimize the sum (or average) of static (deterministic) local objectives, with each agent only aware of its local objective function. Our formulation and results may be viewed as an extension of the above to dynamic uncertain scenarios, in which the environmental dynamics is modeled as a finite-state Markov chain, and, instead of optimizing over a static variable, the agents are interested in obtaining a control policy that minimizes a long-term running cost. Further, in contrast to the static distributed optimization scenarios, the current formulation assumes no prior information on the statistics of the local one-stage costs and the transition probabilities of the controlled state process; instead learns them from sequentially sensed data (costs).

The rest of the paper is organized as follows. Section~\ref{notgraph} sets notation to be used in the sequel. The multi-agent learning setup is formulated in Section~\ref{sys-model}. Section~\ref{alg_DQ} presents the proposed distributed version of $Q$-learning, $\mathcal{QD}$-learning, in which we also formalize our assumptions on the system model and inter-agent communication. Intermediate results on the properties of distributed and mixed time-scale stochastic recursions are presented in Section~\ref{sec:int_app}, whereas, Section~\ref{sec:conv} is devoted to the convergence analysis of $\mathcal{QD}$-learning and the proof of the main result of the paper as stated in Section~\ref{alg_DQ}. Simulation studies comparing the convergence rate of the proposed distributed learning scheme with that of centralized $Q$-learning are presented in Section~\ref{sec:sim}. Finally, Section~\ref{conclusion} concludes the paper and discusses avenues for further research.

\subsection{Notation}
\label{notgraph} We denote the $k$-dimensional Euclidean space by
$\mathbb{R}^{k}$. The set of reals is denoted by $\mathbb{R}$, whereas, $\mathbb{R}_{+}$ denotes the non-negative reals. The partial order on $\mathbb{R}^{k}$ induced by component-wise ordering will be denoted by $\leqc$, i.e., for $\mathbf{x}$ and $\mathbf{y}$ in $\mathbb{R}^{k}$, the notation $\mathbf{x}\leqc\mathbf{y}$ will be used to indicate that each component of $\mathbf{x}$ is less than or equal to the corresponding of $\mathbf{y}$. The set of $k\times k$ real matrices is denoted by $\mathbb{R}^{k\times k}$. The corresponding subspace of symmetric matrices is denoted by $\mathbb{S}^{k}$. The cone of positive semidefinite matrices is denoted by $\mathbb{S}_{+}^{k}$, whereas, $\mathbb{S}_{++}^{k}$ denotes the subset of positive definite matrices. The $k\times k$ identity matrix is
denoted by $I_{k}$, while $\mathbf{1}_{k}$ and $\mathbf{0}_{k}$ denote
respectively the column vector of ones and zeros in
$\mathbb{R}^{k}$. Often the symbol $0$ is used to denote the $k\times p$ zero matrix, the dimensions being clear from the context. The operator
$\left\|\cdot\right\|$ applied to a vector denotes the standard
Euclidean $\mathcal{L}_{2}$ norm, while applied to matrices denotes the induced
$\mathcal{L}_{2}$ norm, which is equivalent to the matrix spectral radius for symmetric
matrices. The $\mathcal{L}_{\infty}$ norm for vectors and matrices is denoted by $\left\|\cdot\right\|_{\infty}$. For a matrix $A\in\mathbb{S}^{k}$, the ordered eigenvalues will be denoted by $\lambda_{1}(A)\leq\lambda_{2}(A)\leq\cdots\leq\lambda_{k}(A)$. The notation $A\otimes B$, whenever applicable, is used to denote the Kronecker product of matrices $A$ and $B$.

Time is assumed to be discrete or slotted throughout the paper. We reserve the symbols $t$ and $s$ to denote time, $\mathbb{T}_{+}$ denoting the discrete index set $\{0,1,2,\cdots\}$.

Throughout, we will assume the existence of a probability space $(\Omega,\mathcal{F})$ that is rich enough to support all the proposed random objects. For an event $\mathcal{B}\in\mathcal{F}$, the notation $\mathbb{I}(\mathcal{B})$ will be used to denote the corresponding indicator random variable, i.e., $\mathbb{I}(\mathcal{B})$ takes the value one on the event $\mathcal{B}$ and zero otherwise. Probability and expectation on $(\Omega,\mathcal{F})$ will be denoted by $\mathbb{P}(\cdot)$ and $\mathbb{E}[\cdot]$, respectively. All inequalities involving random objects are to be interpreted almost surely (a.s.), unless stated otherwise.

\textbf{Spectral graph theory}: The inter-agent communication topology may be described by an \emph{undirected} graph $G=(V,E)$, with $V=\left[1\cdots N\right]$ and~$E$ denoting the set of agents (nodes) and communication links (edges) respectively. The unordered pair $(n,l)\in E$ if there exists an edge between nodes~$n$ and~$l$. We only consider simple graphs, i.e., graphs devoid of self-loops and multiple edges. A graph is connected if there exists a path\footnote{A path between nodes $n$ and $l$ of length $m$ is a sequence
$(n=i_{0},i_{1},\cdots,i_{m}=l)$ of vertices, such that, $(i_{k},i_{k+1})\in E$ for all $0\leq k\leq m-1$.}, between each pair of nodes. The neighborhood of node~$n$ is
\begin{equation}
\label{def:omega} \Omega_{n}=\left\{l\in V\,|\,(n,l)\in
E\right\} 
\end{equation}
Node~$n$ has degree $d_{n}=|\Omega_{n}|$ (number of edges with~$n$ as one end point.) The structure of the graph can be described by the symmetric $N\times N$ adjacency matrix, $A=\left[A_{nl}\right]$, $A_{nl}=1$, if $(n,l)\in E$, $A_{nl}=0$, otherwise. Let the degree matrix  be the diagonal matrix $D=\mbox{diag}\left(d_{1}\cdots d_{N}\right)$. By definition, the positive semidefinite matrix $L=D-A$ is called the graph Laplacian matrix. The eigenvalues of $L$ can be ordered as $0=\lambda_{1}(L)\leq\lambda_{2}(L)\leq\cdots\leq\lambda_{N}(L)$, the eigenvector corresponding to $\lambda_{1}(L)$ being $(1/\sqrt{N})\mathbf{1}_{N}$. The multiplicity of the zero eigenvalue equals the number of connected components of the network; for a connected graph, $\lambda_{2}(L)>0$. This second eigenvalue is the algebraic connectivity or the Fiedler value of
the network; see \cite{FanChung,Mohar} for detailed treatment of graphs and their spectral theory.

\section{System Model}
\label{sys-model} Let $\{\mathbf{x}_{t}\}$ be a controlled Markov chain taking values in a finite state space $\mathcal{X}=[1,\cdots,M]$. Denoting by $\mathcal{U}$ the set (finite) of control actions $\mathbf{u}$, we assume\footnote{The letters $i$ and $j$ will be reserved mostly to denote a generic element of the state space $\mathcal{X}$, whereas, $u$ will denote a generic element of the control space $\mathcal{U}$. Also, note that the state and control stochastic processes are denoted by bold symbols, $\{\mathbf{x}_{t}\}$ and $\{\mathbf{u}_{t}\}$ respectively, although they assume a finite number of values only.} that the state transition is governed by
\begin{equation}
\label{sys-trans}
\mathbb{P}\left(\mathbf{x}_{t+1}=j~|~\mathbf{x}_{t}=i,\mathbf{u}_{t}=u\right)=p_{i,j}^{u}
\end{equation}
for every $i,j\in\mathcal{X}$ and $u\in\mathcal{U}$, where the state transition probabilities satisfy $\sum_{j\in\mathcal{X}}p_{i,j}^{u}=1$ for all $i\in\mathcal{X}$.

We further assume that there are $N$ agents, with agent $n$ incurring a random one-stage cost\footnote{Note that the instantaneous costs $c_{n}(\cdot)$ depend only on the current state of the process and the control applied, but not on the successor state as is the case with some control problems. However, the latter formulations may often be reduced to the former (i.e., current state and control dependence only) by proper state augmentation.} $c_{n}(i,u)$ whenever control $u$ is applied at state $i$. For a stationary control policy $\pi$, i.e., where $\{\mathbf{u}_{t}\}$ satisfies $\mathbf{u}_{t}=\pi(\mathbf{x}_{t})$ for some function $\pi:\mathcal{X}\mapsto\mathcal{U}$, the state process $\{\mathbf{x}^{\pi}_{t}\}$ (the superscript $\pi$ is used to indicate the dependence on the control policy $\pi$) evolves as a homogenous Markov chain with\footnote{Note that, in general, the set of actions $\mathcal{U}$ is state-dependent, which can be accommodated in our formulation by redefining $\mathcal{U}$ to be the union of all state-dependent action sets and modifying the one-stage costs appropriately.}
\begin{equation}
\label{sys-trans-stat}
\mathbb{P}\left(\mathbf{x}^{\pi}_{t+1}=j~|~\mathbf{x}^{\pi}_{t}=i\right)=p_{i,j}^{\pi(i)}.
\end{equation}
For a stationary policy $\pi$ and initial state $i$ of the process $\{\mathbf{x}^{\pi}_{t}\}$, the infinite horizon discounted cost is given by
\begin{equation}
\label{inf-cost}
V_{i}^{\pi}=\limsup_{T\rightarrow\infty}\mathbb{E}\left[\frac{1}{N}\sum_{n=1}^{N}\sum_{t=0}^{T}\gamma^{t}c_{n}\left(\mathbf{x}^{\pi}_{t},\pi(\mathbf{x}^{\pi}_{t})\right)~|~\mathbf{x}^{\pi}_{0}=i\right],
\end{equation}
where $0<\gamma<1$ is the discounting factor. Note that the cost $V_{i}^{\pi}$, defined as such, is a global (centralized) cost, as it involves the one-stage costs of all the agents. The Markov decision problem (MDP) that we consider in this paper concerns the evaluation of the optimal infinite horizon discounted cost
\begin{equation}
\label{opt-cost}
V^{\ast}_{i}=\inf_{\pi}V_{i}^{\pi}
\end{equation}
and the associated stationary policy $\pi^{\ast}$, provided the latter exists.

Let $\mathbf{V}^{\ast}\in\mathbb{R}^{M}$ denote $[V^{\ast}_{1},\cdots,V^{\ast}_{M}]^{T}$. Denote by $\mathcal{T}:\mathbb{R}^{M}\mapsto\mathbb{R}^{M}$ the (centralized) dynamic programming operator with
\begin{equation}
\label{def_T} \mathcal{T}_{i}(\mathbf{V})=\min_{u\in\mathcal{U}}\left\{\frac{1}{N}\sum_{n=1}^{N}\mathbb{E}[c_{n}(i,u)]+\gamma\sum_{j\in\mathcal{X}}p_{i,j}^{u}V_{j}\right\},
\end{equation}
$\mathcal{T}_{i}(\cdot)$ denoting the $i$-th component functional of $\mathcal{T}(\cdot)$, such that, $\mathcal{T}(\mathbf{V})=[\mathcal{T}_{1}(\mathbf{V}),\cdots,\break\mathcal{T}_{M}(\mathbf{V})]^{T}$ for each $\mathbf{V}\in\mathbb{R}^{M}$. The Bellman equation~\cite{Bertsekas-DP} asserts that $\mathbf{V}^{\ast}$ is a fixed point of $\mathcal{T}(\cdot)$, i.e., $\mathcal{T}(\mathbf{V}^{\ast})=\mathbf{V}^{\ast}$. Further, for discounting factors $\gamma$ that are strictly less than one, it may be readily seen~\cite{Bertsekas-DP} that the dynamic programming operator $\mathcal{T}(\cdot)$ is a strict contraction, thus implying the value function $\mathbf{V}^{\ast}$ to be its unique fixed point. As such, starting with an arbitrary initial approximation $\mathbf{V}_{0}\in\mathbb{R}^{M}$, one obtains a sequence of iterates $\{\mathbf{V}_{t}\}$ of $\mathcal{T}(\cdot)$, with $\mathbf{V}_{t}=\mathcal{T}^{t}(\mathbf{V}_{0})$, such that, $\mathbf{V}_{t}\rightarrow\mathbf{V}^{\ast}$ as $t\rightarrow\infty$. The above iterative construction forms the basis of classical policy iteration methods for evaluating the desired value function $V^{\ast}$ (and hence the corresponding optimal policy $\pi^{\ast}(\cdot)$), at least for the considered scenario with $\gamma<1$. However, in doing so, i.e., in constructing successive iterates of $\mathcal{T}(\cdot)$, the value iteration techniques assume that the problem statistics (the expected one-stage costs and the state transition probabilities $p_{i,j}^{u}$) are perfectly known apriori.

Reinforcement learning methods are motivated by scenarios involving lack of information about the problem statistics. Based on a reformulation of the Bellman equation, $\mathcal{T}(\mathbf{V}^{\ast})=\mathbf{V}^{\ast}$, the class of $Q$-learning methods generate sequential (stochastic) approximations of the value function\footnote{To be precise, as will be shown later, instead of generating successive approximations of the state-value function $V^{\ast}_{i}$, $i\in\mathcal{X}$, $Q$-learning methods generate approximations of the so-called state-action value functions $Q^{\ast}_{i,u}$, $(i,u)\in\mathcal{X}\times\mathcal{U}$, (often known as the $Q$-matrices or factors) from which the desired value functions may be recovered.} using instantiations of state-action trajectories, as opposed to relying on exact problem statistics. The state-action trajectory instantiations for value function learning may correspond to online real-time data obtained while implementing the control, in which case the resulting $Q$-learning methods are, in fact, instances of direct adaptive control~\cite{Sutton-Q}, or, may correspond to offline training data obtained through simulated state-action responses. As far as analysis is concerned, the former subsumes the latter, as trajectories that are obtained in the process of real-time control implementation incur temporal statistical dependencies due to memory in the sequential control selection task. While the $Q$-learning techniques discussed above are appealing as they relax the requirement of prior system model knowledge, in the context of our multi-agent setting, they rely on a centralized architecture that requires the instantaneous agent one-stage costs $c_{n}(\mathbf{x}_{t},\mathbf{u}_{t})$ (for each network agent $n$) to be available at a centralized computing resource at all times $t$ with a view to obtaining an approximation of the sum of expectations in~\eqref{def_T}. Since, the instantaneous one-stage costs may only be observed at the agents, this, in turn, requires each network agent to transmit its one-stage cost to the remote central location at all times, which may not be feasible due to limited energy resources at the agents and a bit-budgeted communication medium. This motivates us to consider a fully distributed alternative, in which the agents autonomously engage in the learning process through collaborative local communication and computation.

\section{$\mathcal{QD}$-learning: Distributed Collaborative $Q$-Learning}
\label{alg_DQ} In this section, we present a distributed scheme for multi-agent $Q$-learning, the $\mathcal{QD}$-learning. Like its centralized counterpart, $\mathcal{QD}$-learning is based on instantiations of state-action trajectories. In general, the state-action trajectories are sample paths of stochastic processes $\{\mathbf{x}_{t}\}$ and $\{\mathbf{u}_{t}\}$ taking values in $\mathcal{X}$ and $\mathcal{U}$, respectively. In addition, we have the local one-stage cost processes, $\{c_{n}(\mathbf{x}_{t},\mathbf{u}_{t})\}$ for each agent $n$, as a result of the randomly generated actions $\mathbf{u}_{t}$ and states $\mathbf{x}_{t}$ that are accessible to the corresponding agents. The goal of $\mathcal{QD}$-learning scheme is to ensure that each agent eventually learns the value function $\mathbf{V}^{\ast}$ based on the stochastic processes $\{\mathbf{x}_{t}\}$, $\{\mathbf{u}_{t}\}$, and the one-stage cost processes. To formalize the distributed agent learning, we impose the following measurability requirements that characterize the locally accessible agent information over time for decision-making.

\textbf{(M.1)}: \emph{There exists a complete probability space $\left(\Omega,\mathcal{F},\mathbb{P}\right)$ with a filtration $\{\mathcal{F}_{t}\}$, such that the state and control processes, $\{\mathbf{x}_{t}\}$ and $\{\mathbf{u}_{t}\}$, respectively, are adapted to $\mathcal{F}_{t}$. The conditional probability law governing the controlled transitions of $\{\mathbf{x}_{t}\}$ satisfies
\begin{equation}
\label{M1state}
\mathbb{P}\left(\mathbf{x}_{t+1}=j~|~\mathcal{F}_{t}\right)=p^{\mathbf{u}_{t}}_{\mathbf{x}_{t},j},
\end{equation}
and we require for each $n$
\begin{equation}
\label{M1cost}
\mathbb{E}\left[c_{n}(\mathbf{x}_{t},\mathbf{u}_{t})~|~\mathcal{F}_{t}\right]=\mathbb{E}\left[c_{n}(\mathbf{x}_{t},\mathbf{u}_{t})~|~\mathbf{x}_{t},\mathbf{u}_{t}\right],
\end{equation}
which equates to $\mathbb{E}[c_{n}(i,u)]$ on the event $\{\mathbf{x}_{t}=i,\mathbf{u}_{t}=u\}$, i.e., conditioned on the current state-action pair, the one-stage random costs are independent of $\mathcal{F}_{t}$. Further, we assume that the random cost $c_{n}(\mathbf{x}_{t},\mathbf{u}_{t})$ is adapted to $\mathcal{F}_{t+1}$ for each $t$. Note that~\eqref{M1state}-\eqref{M1cost} is a formal restatement of the fact that the online state-action trajectories and the associated costs that are to be used for value function learning, satisfy the controlled Markov transitions in accordance with the MDP. The obvious choice of such a filtration would be the natural one induced by the processes\footnote{For a collection $\mathcal{J}$ of random objects, the notation $\sigma(\mathcal{J})$ denotes the smallest $\sigma$-algebra with respect to which all the random objects in $\mathcal{J}$ are measurable.}, i.e.,
\begin{equation}
\label{M1filt}\mathcal{F}_{t}=\sigma\left(\{\mathbf{x}_{s},\mathbf{u}_{s}\}_{s\leq t},\{c_{n}(\mathbf{x}_{s},\mathbf{u}_{s})\}_{n\in\mathbb{N},s<t}\right),
\end{equation}
provided the state-actions are generated according to the given MDP dynamics.}

\emph{Also, we assume that the one-stage random costs possess super-quadratic moments, i.e., in particular, we assume there exists a constant $\Vap_{1}>0$ (could be arbitrarily small) such that
\begin{equation}
\label{M1moment}
\mathbb{E}\left[c_{n}^{2+\Vap_{1}}(i,u)\right]<\infty
\end{equation}
for all $n$, $i$, and $u$.}

Note that $\mathcal{F}_{t}$, as defined above, represents the global network information at each time instant $t$. In the sequel, we will also need to characterize the local information $\mathcal{F}_{n}(t)$ available at each agent $n$ at time $t$ on which the agent's instantaneous local decision-making is based. The local information at an agent $n$ reflects its locally sensed cost data and the messages or information it obtains from its neighbors over time among other locally observed variables, such as the instantaneous state and control data. To formalize, let $m_{n,l}(t)$ denote the message that agent $n$ obtains from its neighbor $l\in\Omega_{n}(t)$ at time $t$, where $\Omega_{n}(t)$ denotes the time-varying (possibly stochastic) communication neighborhood of agent $n$ at time $t$. The local information $\mathcal{F}_{n}(t)$ at agent $n$ at time $t$ (up to the beginning of time slot $t+1$) is then formally represented by the $\sigma$-algebra
\begin{equation}
\label{def_locinf}\mathcal{F}_{n}(t)=\sigma\left(\left\{\mathbf{x}_{s},\mathbf{u}_{s}, \{m_{n,l}(s)\}_{l\in\Omega_{n}(s)}\right\}_{s\leq t}, \left\{c_{n}(\mathbf{x}_{s},\mathbf{u}_{s})\right\}_{s<t}\right).
\end{equation}
Further, for the inter-agent message exchange process to be consistent with respect to (w.r.t.) the local information sequences $\{\mathcal{F}_{n}(t)\}$, we require
\begin{equation}
\label{req_loc_mes} m_{l,n}(t)\in\mathcal{F}_{n}(t)
\end{equation}
for each pair of agents $(n,l)$, such that, $n\in\Omega_{l}(t)$ at all times $t$. The key difference between the global network information $\mathcal{F}_{t}$ (as would be available to a fictitious center for decision-making) and the local agent information $\mathcal{F}_{n}(t)$ is in terms of accessibility of the reward information -- the latter consists of only the locally sensed reward data, whereas the former involves the sum-total network reward information from all agents at all times. The lack of global information at the local agent level justifies the need for collaboration, in which the agents engage in mutual neighborhood message exchanges with a view to eventually disseminating the required reward statistics across the network. With the above formalism of distributed collaboration, in particular~\eqref{def_locinf}-\eqref{req_loc_mes}, it is readily seen that (as expected),
\begin{equation}
\label{loc-glob-filt}
\mathcal{F}_{t}=\bigvee_{n=1}^{N}\mathcal{F}_{n}(t),
\end{equation}
where $\bigvee$ denotes the `join' of $\sigma$-algebras, i.e., the global information at an instant $t$ is the sum-total of the local agent information, provided $\mathcal{F}_{t}$ corresponds to the natural filtration induced by the state-action pairs and the instantaneous rewards as in~\eqref{M1filt}. Moreover, in general, we have $\mathcal{F}_{n}(t)\subset\mathcal{F}_{t}$ for each $n$ and $t$, the inclusion being strict usually if the inter-agent communication graph is not complete. In general, we are interested in applications with sparse inter-agent connectivity in which, even with agent collaboration, the local information sets $\mathcal{F}_{n}(t)$ are strict subsets of the global $\mathcal{F}_{t}$ as explained above, and the fundamental goal of this paper is to design distributed message exchange and local processing policies that in the long-run lead to sufficient network-wide information dissemination, such that, each agent eventually obtains an accurate estimate of the desired value function $\mathbf{V}^{\ast}$. As will be seen, a necessary condition for successful eventual information dissemination involves \emph{long-term connectivity} of the inter-agent communication graph. To this end, we assume that the time-varying stochastic inter-agent communication graphs (generating the neighborhoods $\Omega_{n}(t)$ for each agent $n$ at every instant $t$) satisfies the following weak connectivity condition:

\textbf{(M.2)}:  \emph{To account for possible random packet losses or infrastructure failures, as is commonly encountered in wireless multi-agent communication settings, we assume that the agent network at time~$t$ is modeled as
an undirected graph, $G_{t}=(V,E_{t})$, with the graph Laplacians being a
sequence of i.i.d.~Laplacian matrices $\{L_{t}\}$. Specifically, we assume that $L_{t}$ is $\mathcal{F}_{t+1}$ adapted and is independent of $\mathcal{F}_{t}$. We do
not make any distributional assumptions on the link failure model.
Although the link failures, and so the Laplacians, are independent
at different times, during the same iteration, the link failures
can be spatially dependent, i.e., correlated. This is more general
and subsumes the erasure network model, where the link failures
are independent over space \emph{and} time. Wireless agent
networks motivate this model since interference among the wireless
communication channels correlates the link failures over space,
while, over time, it is still reasonable to assume that the
channels are memoryless or independent. Finally, note that we do not
require that the random instantiations~$G_{t}$ of the graph be
connected; in fact, it is possible to have all these
instantiations to be disconnected. We only require that the graph
stays connected on \emph{average}. Denoting $\mathbb{E}[L_{t}]$ by $\overline{L}$, this is captured by assuming
$\lambda_{2} \left(\overline{L}\right)>0$. This weak connectivity requirement enables us to
capture a broad class of asynchronous communication models; for
example, the random asynchronous gossip protocol analyzed
in~\cite{Boyd-Gossip} satisfies $\lambda_{2}\left(\overline{L}\right)>0$ and hence falls under this framework. On the other hand, we assume that the inter-agent communication is noise-free and unquantized in the event of an active communication link; the problem of quantized data exchange in networked control systems (see, for example,~\cite{tm04,ms03,Li-Baillieul,karmoura-quantized}) is an active research topic.}

\textbf{(M.3)}: \emph{At each $t$, the Laplacian $L_{t}$ is assumed to be independent of the instantaneous costs $c_{n}(\mathbf{x}_{t},\mathbf{u}_{t})$ conditioned on the state-action action pair $(\mathbf{x}_{t},\mathbf{u}_{t})$.}

We now consider $\mathcal{QD}$-learning, in which network agents engage in mutual collaboration with a view to learning the true value function $\mathbf{V}^{\ast}$ eventually.

Before presenting the distributed update rule, for each pair $(i,u)$, let us introduce the sequence of random times $\{T_{i,u}(k)\}$, such that, $T_{i,u}(k)$ denotes the $(k+1)$-th sampling instant of the state-action pair $(i,u)$, i.e.,
\begin{equation}
\label{def_Tiu}T_{i,u}(k)=\inf\left\{t\geq 0~|~\sum_{s=0}^{t}\mathbb{I}_{(\mathbf{x}_{t},\mathbf{u}_{t})=(i,u)}=k+1\right\},
\end{equation}
for each $k\geq 0$, in which we adopt the convention that the infimum of an empty set is $\infty$. It can be shown that the random time $T_{i,u}(k)$, for each $k$ and pair $(i,u)$, is a stopping time w.r.t. the filtration $\{\mathcal{F}_{t}\}$. Further, note that, since we assume that the state-action pairs $(\mathbf{x}_{t},\mathbf{u}_{t})$ are accessible to the agents also, see~\eqref{def_locinf}, $T_{i,u}(k)$, for each $k$, qualifies as a stopping time w.r.t the local filtrations $\{\mathcal{F}_{n}(t)\}$ as well. The following requirement that ensures each state-action pair $(i,u)$ is observed (simulated) infinitely often is imposed:

\textbf{(M.4)}: \emph{For each state-action pair $(i,u)$ and each $k\geq 0$, the stopping time $T_{i,u}(k)$ is a.s. finite, i.e.,}
\begin{equation}
\label{cond-stopping}\mathbb{P}\left(T_{i,u}(k)<\infty\right)=1.
\end{equation}
It is to be noted, that~\textbf{(M.4)} is required in all forms of centralized $Q$-learning, either real-time direct adaptive control based or simulation based approaches, for desired convergence with generic initial conditions (approximations).

\textbf{$\mathcal{QD}$-learning}: In $\mathcal{QD}$-learning, each network agent $n$ maintains a $\mathbb{R}^{|\mathcal{X}\times\mathcal{U}|}$-valued sequence $\{\mathbf{Q}^{n}_{t}\}$ (approximations of the so-called $Q$ matrices) with components $Q^{n}_{i,u}(t)$ for every possible state-action pair $(i,u)$.
With this, the sequence $\{Q^{n}_{i,u}(t)\}$ at each agent $n$ for each pair $(i,u)$ evolves in a collaborative distributed fashion as follows:
\begin{align}
\label{up_Q} Q^{n}_{i,u}(t+1)=Q^{n}_{i,u}(t)-\beta_{i,u}(t)\sum_{l\in\Omega_{n}(t)}\left(Q^{n}_{i,u}(t)-Q^{l}_{i,u}(t)\right)\\+\alpha_{i,u}(t)\left(c_{n}(\mathbf{x}_{t},\mathbf{u}_{t})+\gamma\min_{v\in\mathcal{U}}Q^{n}_{\mathbf{x}_{t+1},v}(t)-Q^{n}_{i,u}(t)\right),
\end{align}
where the weight sequences $\{\beta_{i,u}(t)\}$ and $\{\alpha_{i,u}(t)\}$ are $\mathcal{F}_{n}(t)$-adapted stochastic processes for each pair $(i,u)$ and given by
\begin{equation}
\label{def_beta}
\beta_{i,u}(t)=\left\{
\begin{array}{cc}
                    \frac{b}{(k+1)^{\tau_{2}}} & \mbox{if $t=T_{i,u}(k)$ for some $k\geq 0$} \\
                    0 & \mbox{otherwise},
                   \end{array}
          \right.
\end{equation}
and
\begin{equation}
\label{def_alpha}
\alpha_{i,u}(t)=\left\{
\begin{array}{cc}
                    \frac{a}{(k+1)^{\tau_{1}}} & \mbox{if $t=T_{i,u}(k)$ for some $k\geq 0$} \\
                    0 & \mbox{otherwise},
                   \end{array}
          \right.
\end{equation}
$a$ and $b$ being positive constants. In other words, as reflected by the weight sequences~\eqref{def_beta}-\eqref{def_alpha}, at each agent $n$, the component $Q^{n}_{i,u}(t)$ is updated at an instant\footnote{Note that the phrase \emph{updated at an instant $t$} refers to the possible transition of $Q^{n}_{i,u}(t)$ to $Q^{n}_{i,u}(t+1)$, an event which actually occurs after the one-stage cost $c_{n}(\mathbf{x}_{t},\mathbf{u}_{t})$ has been incurred and the successor state $\mathbf{x}_{t+1}$ has been reached. In terms of implementation, such an update may be realized at the end of the time slot $t$.} $t$ \emph{iff} the current state-action pair $(\mathbf{x}_{t},\mathbf{u}_{t})$ corresponds to $(i,u)$, and otherwise stays constant.

In addition to the processes $\{\mathbf{Q}^{n}_{t}\}$, each agent $n$ maintains another $\{\mathcal{F}_{n}(t)\}$-adapted $\mathbb{R}^{|\mathcal{X}|}$-values process $\{\mathbf{V}^{n}_{t}\}$, that serves as an approximation of the desired value function $\mathbf{V}^{\ast}$. The $i$-th component of $\mathbf{V}^{n}_{t}$, $V^{n}_{i}(t)$, is successively refined as
\begin{equation}
\label{up_V} V^{n}_{i}(t)=\min_{u\in\mathcal{U}}Q^{n}_{i,u}(t),
\end{equation}
for $i=1,\cdots,M$.

The $\{\mathcal{F}_{n}(t)\}$-adaptability of the weight sequences, for each $n$, follows from the fact that the random times $T_{i,u}(k)$, for all $k$, are stopping times w.r.t. $\{\mathcal{F}_{n}(t)\}$. With the identification that
\begin{equation}
\label{mess_id} m_{n,l}(t)
\doteq \mathbf{Q}^{l}_{t}~~\mbox{if $l\in\Omega_{n}(t)$},
\end{equation}
where $m_{n,l}(t)$ denotes the message sent to agent $n$ by agent $l$ at time $t$, it is readily seen that, for each $n$, the process $\{\mathbf{Q}^{n}_{t}\}$ is well-defined and adapted to the local filtration $\{\mathcal{F}_{n}(t)\}$. We note that the update rule in~\eqref{up_Q} is of the consensus + innovations form, in that it consists of the interplay between an agreement or consensus potential reflecting agent collaboration, and a local innovation potential that involves the incorporation of newly obtained intelligence through local sensing of the instantaneous cost. The convergence of the resulting algorithm may only be achieved by intricately trading off these potentials, which, in turn, imposes further restrictions on the algorithm weight sequences as follows:

\textbf{(M.5)}: \emph{The constants $\tau_{1}$ and $\tau_{2}$ in~\eqref{def_beta}-\eqref{def_alpha} are assumed to satisfy $\tau_{1}\in (1/2, 1]$ and $0<\tau_{2}<\tau_{1}-1/(2+\Vap_{1})$, with $\Vap_{1}$ being defined in~\eqref{M1moment}. The above together with assumption~\textbf{(M.4)} guarantee that the excitations from the consensus and innovation potentials are persistent, i.e., the (stochastic) sequences $\{\alpha_{i,u}(t)\}$ and $\{\beta_{i,u}(t)\}$ sum to $\infty$, for each state-action pair $(i,u)$. They further guarantee that the innovation weight sequences are square summable, i.e., $\sum_{t\geq 0}\alpha_{i,u}^{2}(t)<\infty$ a.s., and that, the consensus potential dominates the innovation potential eventually, i.e., $\beta_{i,u}(t)/\alpha_{i,u}(t)\rightarrow\infty$ a.s. as $t\rightarrow\infty$ for each pair $(i,u)$.}

\begin{remark}
\label{rem:weights} 
We comment on the choice of the weight sequences $\{\beta_{i,u}(t)\}$ and $\{\alpha_{i,u}(t)\}$ associated with the consensus and innovation potentials respectively. From~\textbf{(M.5)} (and~\textbf{(M.4)}) we note that both the excitations for agent-collaboration (consensus) and local innovation are persistent, i.e., the sequences $\{\beta_{i,u}(t)\}$ and $\{\alpha_{i,u}(t)\}$ sum to $\infty$ -  a standard requirement in stochastic approximation type algorithms to drive the updates to the desired limit from arbitrary initial conditions. Further, the square summability of $\{\alpha_{i,u}(t)\}$ ($\tau_{1}>1/2$) is required to mitigate the effect of stochasticity (due to the randomness involved in the one-stage costs and the state transitions) the innovations. The requirement $\beta_{i,u}(t)/\alpha_{i,u}(t)\rightarrow\infty$ as $t\rightarrow\infty$ ($\tau_{1}>\tau_{2}$), i.e., the asymptotic domination of the consensus potential over the local innovations ensures the right information mixing thus, as shown below, leading to \emph{optimal} convergence. Technically, the different asymptotic decay rates of the two potentials lead to mixed time-scale stochastic recursions whose analyses require new techniques in stochastic approximation as developed in the paper.

We further comment on the constants $a$ and $b$ in~\eqref{def_beta}-\eqref{def_alpha} affecting the weight sequences. While the main results and the proof arguments in this paper will continue to hold for arbitrary positive constants $a$ and $b$, to simplify the exposition that follows we further assume that the constants are small enough, such that, for each time instant $t$ and state-action pair $(i,u)$, the matrix $\left(I_{N}-\beta_{i,u}(t)L_{t}-\alpha_{i,u}(t)I_{N}\right)$ is non-negative definite. Noting that the largest eigenvalue of the Laplacian $L_{t}$, at an instant $t$, is upper-bounded by $N$, the number of network agents, the above condition is ensured by requiring $a$ and $b$ to satisfy $a+Nb\leq 1$. We emphasize that the above requirement on $a$ and $b$ is by no means necessary, but greatly reduces the analytical overhead. In fact, for arbitrary positive $a$ and $b$, ~\textbf{(M.4)}-\textbf{(M.5)} imply that, for each state-action pair $(i,u)$, there exists $t_{0}(i,u)\geq 0$ (possibly random), such that the matrix $\left(I_{N}-\beta_{i,u}(t)L_{t}-\alpha_{i,u}(t)I_{N}\right)$ is non-negative definite for $t\geq t_{0}(i,u)$.
\end{remark}

The rest of the paper is devoted to the convergence analysis of the proposed $\mathcal{QD}$-learning, in which our goal is to show that, for each $n$, $\mathbf{V}^{n}_{t}\rightarrow V^{\ast}$ a.s. as $t\rightarrow\infty$, so that eventually each agent obtains the accurate value function and the corresponding optimal stationary strategy (through~\eqref{up_V}). To this end, for each $n$ define the \emph{local} $\mathcal{QD}$-learning operator $\mathcal{G}^{n}:\RXU\mapsto\RXU$ whose components $\mathcal{G}^{n}_{i,u}:\RXU\mapsto\mathbb{R}$ are given by
\begin{equation}
\label{def_Giu}\mathcal{G}^{n}_{i,u}(Q)=\mathbb{E}\left[c_{n}(i,u)\right]+\gamma\sum_{j\in\mathcal{X}}p_{i,j}^{u}\min_{v\in\mathcal{U}}Q_{j,v},
\end{equation}
for all $Q=\{Q_{i,u}\}\in\RXU$. Noting that under~\textbf{(M.1)}, on $\{\mathbf{x}_{t}=i,\mathbf{u}_{t}=u\}$,
\begin{equation}
\label{def_Giu1}\mathbb{E}\left[c_{n}(\mathbf{x}_{t},\mathbf{u}_{t})~|~\mathcal{F}_{t}\right]=\mathbb{E}\left[c_{n}(i,u)\right],
\end{equation}
and
\begin{equation}
\label{def_Giu2}
\mathbb{E}\left[\min_{v\in\mathcal{U}}Q^{n}_{\mathbf{x}_{t+1},v}(t)~|~\mathcal{F}_{t}\right]=\sum_{j\in\mathcal{X}}p_{i,j}^{u}\min_{v\in\mathcal{U}}Q^{n}_{j,v}(t),
\end{equation}
the recursive update in~\eqref{up_Q}, for each state-action pair $(i,u)$, may be rewritten as
\begin{align}
\label{def_Giu3}
Q^{n}_{i,u}(t+1)=Q^{n}_{i,u}(t)-\beta_{i,u}(t)\sum_{l\in\Omega_{n}(t)}\left(Q^{n}_{i,u}(t)-Q^{l}_{i,u}(t)\right)\\+\alpha_{i,u}(t)\left(\mathcal{G}^{n}_{\mathbf{x}_{t},\mathbf{u}_{t}}(\mathbf{Q}^{n}_{t})-Q^{n}_{\mathbf{x}_{t},\mathbf{u}_{t}}(t)+\mathbf{\nu}^{n}_{\mathbf{x}_{t},\mathbf{u}_{t}}(t)\right),
\end{align}
in which the residual
\begin{equation}
\label{def_Giu4}
\mathbf{\nu}^{n}_{\mathbf{x}_{t},\mathbf{u}_{t}}(t)=c_{n}(\mathbf{x}_{t},\mathbf{u}_{t})+\gamma\min_{v\in\mathcal{U}}Q^{n}_{\mathbf{x}_{t+1},v}(t)-\mathcal{G}^{n}_{\mathbf{x}_{t},\mathbf{u}_{t}}(\mathbf{Q}^{n}_{t})
\end{equation}
plays the role of a martingale difference noise, i.e., $\mathbb{E}[\mathbf{\nu}^{n}_{\mathbf{x}_{t},\mathbf{u}_{t}}(t)~|~\mathcal{F}_{t}]=\mathbf{0}$ for all $t$.

\subsection{Main Result}
\label{subsec:main_res} The main result of the paper concerning the convergence of the proposed $\mathcal{QD}$-learning is stated as follows (proof provided in Section~\ref{sub:proof}):
\begin{theorem}
\label{th:main_res} Let $\{\mathbf{Q}^{n}_{t}\}$ and $\{\mathbf{V}^{n}_{t}\}$ be the successive iterates obtained at agent $n$ through the $\mathcal{QD}$-learning (see~\eqref{up_Q} and~\eqref{up_V}). Then, under~\textbf{(M.1)}-\textbf{(M.5)}, there exists $\mathbf{Q}^{\ast}\in\mathbb{R}^{|\mathcal{X}\times\mathcal{U}|}$, such that,
\begin{equation}
\label{th:main_res1}
\mathbb{P}\left(\lim_{t\rightarrow\infty}\mathbf{Q}^{n}_{t}=\mathbf{Q}^{\ast}\right)=1
\end{equation}
for each network agent $n$.
Further, for each $i\in\mathcal{X}$, we have
\begin{equation}
\label{th:main_res2}\min_{u\in\mathcal{U}}Q^{\ast}_{i,u}=V^{\ast}_{i},
\end{equation}
and, hence, in particular, $\mathbf{V}^{n}_{t}\rightarrow\mathbf{V}^{\ast}$ as $t\rightarrow\infty$ a.s. for each $n$, where $\mathbf{V}^{\ast}$ denotes the desired value function~\eqref{opt-cost}.
\end{theorem}

\section{Intermediate Approximation Results}
\label{sec:int_app} This section provides some approximation results to be used in the sequel for the analysis of $\mathcal{QD}$-learning. In what follows, $\{\mathbf{z}_{t}\}$ will denote a stochastic process that is adapted to a generic filtration $\{\mathcal{H}_{t}\}$ (possibly different from $\{\mathcal{F}_{t}\}$) defined on the probability space $\left(\Omega,\mathcal{F},\mathcal{P}\right)$.

The following result from~\cite{SICON-Kar-Moura-Poor} will be used.

\begin{lemma}[Lemma 4.3 in~\cite{SICON-Kar-Moura-Poor}]
\label{lm:mean-conv} Let $\{z_{t}\}$ be an $\mathbb{R}_{+}$ valued $\{\mathcal{H}_{t}\}$ adapted process that satisfies
\begin{equation}
\label{lm:mean-conv1}
z_{t+1}\leq \left(1-r_{1}(t)\right)z_{t}+r_{2}(t)U_{t}\left(1+J_{t}\right).
\end{equation}
In the above, $\{r_{1}(t)\}$ is an $\{\mathcal{H}_{t+1}\}$ adapted process, such that, for all $t$, $r_{1}(t)$ satisfies $0\leq r_{1}(t)\leq 1$ and
\begin{equation}
\label{lm:JSTSP2}
\frac{a_{1}}{(t+1)^{\delta_{1}}}\leq\mathbb{E}\left[r_{1}(t)~|~\mathcal{H}_{t}\right]\leq 1
\end{equation}
with $a_{1}>0$ and $0\leq \delta_{1}< 1$, whereas, the sequence $\{r_{2}(t)\}$ is deterministic, $\mathbb{R}_{+}$ valued, and satisfies $r_{2}(t)\leq a_{2}/(t+1)^{\delta_{2}}$ with $a_{2}>0$ and $\delta_{2}>0$.
Further, let $\{U_{t}\}$ and $\{J_{t}\}$ be $\mathbb{R}_{+}$ valued $\{\mathcal{H}_{t}\}$ and $\{\mathcal{H}_{t+1}\}$ adapted processes respectively with $\sup_{t\geq 0}\|U_{t}\|<\infty$ a.s.,  and $\{J_{t}\}$ is i.i.d. with $J_{t}$ independent of $\mathcal{H}_{t}$ for each $t$ and satisfies the moment condition $\mathbb{E}\left[\left\|J_{t}\right\|^{2+\varepsilon_{1}}\right]<\kappa<\infty$ for some $\varepsilon_{1}>0$ and a constant $\kappa>0$. Then, for every $\delta_{0}$ such that
\begin{equation}
\label{lm:mean-conv5}
0\leq\delta_{0}<\delta_{2}-\delta_{1}-\frac{1}{2+\varepsilon_{1}},
\end{equation}
we have $(t+1)^{\delta_{0}}z_{t}\rightarrow 0$ a.s. as $t\rightarrow\infty$.
\end{lemma}

The following result from~\cite{SICON-Kar-Moura-Poor}, which provides a stochastic characterization of the contraction properties of random time-varying graph Laplacian matrices,  will be used to quantify the rate of convergence of distributed vector or matrix valued recursions to their network-averaged behavior.
\begin{definition}
\label{def:cons}
For positive integers $N$ and $P$, denote by $\mathcal{C}$ the consensus subspace of $\mathbb{R}^{NP}$, i.e.,
\begin{equation}
\label{lm:conn200}
\mathcal{C}=\left\{\mathbf{y}\in\mathbb{R}^{NP}~:~\mathbf{y}=\mathbf{1}_{N}\otimes\mathbf{y}^{\prime}~~\mbox{for some $\mathbf{y}^{\prime}\in\mathbb{R}^{P}$}\right\}.
\end{equation}
Let $\PC$ be the orthogonal complement of $\mathcal{C}$ and note that any $\mathbf{y}\in\mathbb{R}^{NP}$ admits the orthogonal decomposition, $\mathbf{y}=\mathbf{y}_{\C}+\mathbf{y}_{\PC}$, with $\mathbf{y}_{\C}$ denoting the consensus subspace projection of $\mathbf{y}$.
\end{definition}

\begin{lemma}[Lemma 4.4 in~\cite{SICON-Kar-Moura-Poor}]
\label{lm:conn} Let $\{\mathbf{z}_{t}\}$ be an $\mathbb{R}^{NP}$ valued $\{\mathcal{H}_{t}\}$ adapted process such that $\mathbf{z}_{t}\in\PC$ (see Definition~\ref{def:cons}) for all $t$. Also, let $\{L_{t}\}$ be an i.i.d. sequence of graph Laplacian matrices that satisfies
\begin{equation}
\label{Lap_cond}
\lambda_{2}(\overline{L})=\lambda_{2}\left(\mathbb{E}[L_{t}]\right)>0,
\end{equation}
with $L_{t}$ being $\mathcal{H}_{t+1}$ adapted and independent of $\mathcal{H}_{t}$ for all $t$. Then, there exists a measurable $\{\mathcal{H}_{t+1}\}$ adapted $\mathbb{R}_{+}$ valued process $\{r_{t}\}$ (depending on $\{\mathbf{z}_{t}\}$ and $\{L_{t}\}$) and a constant $c_{r}>0$, such that $0\leq r_{t}\leq 1$ a.s. and
\begin{equation}
\label{lm:conn20}
\left\|\left(I_{NP}-\bar{r}_{t}L_{t}\otimes I_{P}\right)\mathbf{z}_{t}\right\|\leq\left(1-r_{t}\right)\left\|\mathbf{z}_{t}\right\|
\end{equation}
with
\begin{equation}
\label{lm:conn2}
\mathbb{E}\left[r_{t}~|~\mathcal{H}_{t}\right]\geq\frac{c_{r}}{(t+1)^{\delta}}~~\mbox{a.s.}
\end{equation}
for all $t$, where the weight sequence $\{\bar{r}_{t}\}$ satisfies $\bar{r}_{t}\leq\bar{r}/(t+1)^{\delta}$ for some $\bar{r}>0$ and $\delta\in (0,1]$.
\end{lemma}

For a discussion of the necessary technicalities involved in the construction of the sequence $\{r_{t}\}$, the reader is referred to~\cite{SICON-Kar-Moura-Poor} (Remark 4.1).

\begin{lemma}
\label{lm:int-bound}
For each state-action pair $(i,u)$, let $\{\mathbf{z}_{i,u}(t)\}$ denote the $\{\mathcal{F}_{t}\}$ adapted process evolving as
\begin{equation}
\label{lm:int-bound100}
\mziu(t+1)=\left(I_{N}-\beta_{i,u}(t)L_{t}-\alpha_{i,u}(t)I_{N}\right)\mziu(t)+\alpha_{i,u}(t)\bniu(t),
\end{equation}
where the weight sequences $\{\beta_{i,u}(t)\}$ and $\{\alpha_{i,u}(t)\}$ are given by~\eqref{def_beta}-\eqref{def_alpha} and $\{\bniu(t)\}$ is an $\{\mathcal{F}_{t+1}\}$ adapted process satisfying $\mathbb{E}[\bniu(t)~|~\mathcal{F}_{t}]=\mathbf{0}$ for all $t$ and
\begin{equation}
\label{lm:int-bound200}
\sup_{t\geq 0}\mathbb{E}\left[\left\|\bniu(t)\right\|^{2}~|~\mathcal{F}_{t}\right]<K<\infty,
\end{equation}
$K$ being a constant.
Then, under~\textbf{(M.4)}-\textbf{(M.5)}, we have $\mziu(t)\rightarrow\mathbf{0}$ as $t\rightarrow\infty$ a.s.
\end{lemma}

The following result will be used in the proof of Lemma~\ref{lm:int-bound}.
\begin{proposition}
\label{lm:decsup} Let $\{z_{t}\}$ be a real-valued deterministic process, such that,
\begin{equation}
\label{lm:decsup1}
z_{t+1}\leq\left(1-\alpha_{t}\right)z_{t}+\alpha_{t}\Vap_{t},
\end{equation}
where the deterministic sequences $\{\alpha_{t}\}$ and $\{\Vap_{t}\}$ satisfy $\alpha_{t}\in [0,1]$ for all $t$, $\sum_{t\geq 0}\alpha_{t}=\infty$, and there exists a constant $R>0$, such that,
\begin{equation}
\label{lm:decsup2}\limsup_{t\rightarrow\infty}\Vap_{t}\leq R.
\end{equation}
Then, $\limsup_{t\rightarrow\infty}z_{t}\leq R$.
\end{proposition}

Variants of the above result may be found in the literature. We provide a simple self-contained proof in the following.
\begin{proof}
Consider $\delta>0$ and note that, by hypothesis, there exists $t_{\delta}>0$, such that, $\Vap_{t}\leq (R+\delta)$ for all $t\geq t_{\delta}$. Hence, for $t\geq t_{\delta}$, we have
\begin{equation}
\label{lm:decsup3}
z_{t+1}\leq\left(1-\alpha_{t}\right)z_{t}+\alpha_{t}\left(R+\delta\right).
\end{equation}
Hence, denoting by $\{\widehat{z}_{t}\}$ the sequence with $\widehat{z}_{t}=z_{t}-(R+\delta)$ for all $t$, we have, for $t\geq t_{\delta}$,
\begin{equation}
\label{lm:decsup4}\widehat{z}_{t+1}\leq\left(1-\alpha_{t}\right)\widehat{z}_{t}.
\end{equation}
Since $\sum_{t\geq t_{\delta}}\alpha_{t}=\infty$, we conclude that
\begin{equation}
\label{lm:decsup5}
\limsup_{t\rightarrow\infty}\prod_{s=t_{\delta}}^{t-1}\left(1-\alpha_{s}\right)\leq\limsup_{t\rightarrow\infty}e^{-\sum_{s=t_{\delta}}^{t-1}\alpha_{s}}=0,
\end{equation}
and hence, by~\eqref{lm:decsup4}, $\limsup_{t\rightarrow\infty}\widehat{z}_{t}\leq 0$. We thus obtain
\begin{equation}
\label{lm:decsup6}
\limsup_{t\rightarrow\infty}z_{t}\leq R+\delta,
\end{equation}
from which the desired assertion follows by taking $\delta$ to zero.
\end{proof}

We now complete the proof of Lemma~\ref{lm:int-bound}.
\begin{proof}[Lemma~\ref{lm:int-bound}] Recall the consensus subspace $\mathcal{C}$ of $\mathbb{R}^{N}$ (see Definition~\ref{def:cons} with $P=1$). By properties of the Laplacian, we obtain the following inequalities for each $\mathbf{y}\in\mathbb{R}^{N}$:
\begin{equation}
\label{lm:int-bound1}\lambda_{2}(\OL)\left\|\mathbf{y}_{\PC}\right\|^{2}\leq\mathbf{y}^{T}\OL\mathbf{y}\leq\lambda_{N}(\OL)\left\|\mathbf{y}_{\PC}\right\|^{2},
\end{equation}
and
\begin{equation}
\label{lm:int-bound2}
\mathbf{y}^{T}\mathbb{E}\left[L_{t}^{2}\right]\mathbf{y}\leq c_{1}\left\|\mathbf{y}_{\PC}\right\|^{2}
\end{equation}
for each $t$, where $c_{1}>0$ is a constant. Now consider the $\{\mathcal{F}_{t}\}$ adapted process $\{V_{t}\}$, such that $V_{t}=\|\mziu(t)\|^{2}$ for each $t$, and note that under the hypotheses of Lemma~\ref{lm:int-bound} we have,
\begin{align}
\label{lm:int-bound3}
\mathbb{E}\left[V_{t+1}~|~\mathcal{F}_{t}\right]=V_{t}-2\aiu(t)V_{t}-2\biu(t)\mziu^{T}(t)\OL\mziu(t)\\ +\biu^{2}(t)\mziu^{T}(t)\mathbb{E}[L_{t}^{2}]\mziu(t)
+2\aiu(t)\biu(t)\mziu^{T}(t)\OL\mziu(t)+\aiu^{2}(t)V_{t}\\+\aiu^{2}(t)\mathbb{E}\left[\|\bniu(t)\|^{2}~|~\mathcal{F}_{t}\right]\\
\leq \left(1-2\aiu(t)+\aiu^{2}(t)\right)V_{t}\\
-\left(2\biu(t)\lambda_{2}(\OL)-\biu^{2}(t)c_{1}-2\aiu(t)\biu(t)\lambda_{N}(\OL)\right)\left\|(\mziu(t))_{\PC}\right\|^{2}+\aiu^{2}(t)c_{2},
\end{align}
where $c_{2}>0$ is a constant and in the last step we make use of~\eqref{lm:int-bound1}-\eqref{lm:int-bound2}.

Recall the stopping times $\{T_{i,u}(k)\}$ and note that, by~\eqref{def_beta}-\eqref{def_alpha}, there exists a positive integer $k_{0}$ and a constant $c_{3}>0$, such that $t\geq T_{i,u}(k_{0})$ implies a.s.
\begin{equation}
\label{lm:int-bound4}
0<\left(1-2\aiu(t)+\aiu^{2}(t)\right)\leq\left(1-c_{3}\aiu(t)\right),
\end{equation}
and
\begin{equation}
\label{lm:int-bound5}
2\biu(t)\lambda_{2}(\OL)-\biu^{2}(t)c_{1}-2\aiu(t)\biu(t)\lambda_{N}(\OL)\geq 0.
\end{equation}
By~\textbf{(M.4)}, the stopping time $T_{i,u}(k_{0})$ is finite a.s., and hence, for every $\Vap>0$, there exists $t_{\Vap}>0$ (deterministic), such that
\begin{equation}
\label{lm:int-bound6}
\mathbb{P}\left(T_{i,u}(k_{0})>t_{\Vap}\right)<\Vap.
\end{equation}
Now, for a given $\Vap>0$, construct the process $\{V^{\Vap}_{t}\}$ as follows:
\begin{equation}
\label{lm:int-bound7}
V_{t}^{\Vap}=\mathbb{I}\left(T_{i,u}(k_{0})\leq t_{\Vap}\right)V_{t}~~\forall t.
\end{equation}
Since $\{T_{i,u}(k_{0})\leq t_{\Vap}\}\in\mathcal{F}_{t_{\Vap}}$, we note that $V_{t}^{\Vap}$ is adapted to $\mathcal{F}_{t}$ for all $t\geq t_{\Vap}$. Also, by~\eqref{lm:int-bound3}-\eqref{lm:int-bound5}, for $t\geq t_{\Vap}$, we have
\begin{align}
\label{lm:int-bound8}
\mathbb{E}\left[V_{t+1}^{\Vap}~|~\mathcal{F}_{t}\right]=\mathbb{I}\left(T_{i,u}(k_{0})\leq t_{\Vap}\right)\mathbb{E}\left[V_{t+1}~|~\mathcal{F}_{t}\right]\\
\leq\mathbb{I}\left(T_{i,u}(k_{0})\leq t_{\Vap}\right)\left[\left(1-2\aiu(t)+\aiu^{2}(t)\right)V_{t}\right. \\ \left.
-\left(2\biu(t)\lambda_{2}(\OL)-\biu^{2}(t)c_{1}-2\aiu(t)\biu(t)\lambda_{N}(\OL)\right)\left\|(\mziu(t))_{\PC}\right\|^{2}+\aiu^{2}(t)c_{2}\right]\\
\leq\mathbb{I}\left(T_{i,u}(k_{0})\leq t_{\Vap}\right)\left[\left(1-c_{3}\aiu(t)\right)V_{t}\right]+\aiu^{2}(t)c_{2}\leq\left(1-c_{3}\aiu(t)\right)V_{t}^{\Vap}+\aiu^{2}(t)c_{2}.
\end{align}
With the above, the pathwise instantiations of the process $\{V_{t}^{\Vap}\}$ clearly fall under the purview of Proposition~\ref{lm:decsup}, and we conclude that
\begin{equation}
\label{lm:int-bound9}
\mathbb{P}\left(\lim_{t\rightarrow\infty}V_{t}^{\Vap}=0\right)=1.
\end{equation}
This, together with~\eqref{lm:int-bound7}, implies that the process $\{V_{t}\}$ converges to zero on the event $\{T_{i,u}(k_{0})\leq t_{\Vap}\}$, and, hence, by~\eqref{lm:int-bound6} we obtain
\begin{equation}
\label{lm:int-bound10}
\mathbb{P}\left(\lim_{t\rightarrow\infty}V_{t}=0\right)>1-\Vap.
\end{equation}
Since $\Vap>0$ is arbitrary, the desired result follows by taking $\Vap$ to zero.
\end{proof}
\begin{remark}
\label{rem:lm:int} Note that, although the statement of Lemma~\ref{lm:int-bound} assumes~\textbf{(M.4)}-\textbf{(M.5)} to hold, the only condition on the sequence $\{\biu(t)\}$ that we actually use in the proof involves the requirement that~\eqref{lm:int-bound5} holds eventually. Given that~\eqref{lm:int-bound5} holds trivially for all $t$ if $\biu(t)=0$ for all $t$, we note that the assertions of Lemma~\ref{lm:int-bound} continue to hold if $\{\biu(t)\}$ is set to zero identically (i.e., the Laplacian dependent dynamics is dropped) in the update process~\eqref{lm:int-bound100}.
\end{remark}

\begin{corollary}
\label{corr:int-bound} For each state-action pair $(i,u)$ and $t_{0}\geq 0$, consider the process $\{\mziu(t:t_{0})\}_{t\geq t_{0}}$ that evolves as
\begin{equation}
\label{corr:int-bound1}\mziu(t+1:t_{0})=\left(I_{N}-\biu(t)L_{t}-\aiu(t)I_{N}\right)\mziu(t:t_{0})+\aiu(t)\bniu(t)
\end{equation}
with $\mziu(t_{0}:t_{0})=\mathbf{0}$, where the processes $\{\biu(t)\}$, $\{\aiu(t)\}$, and $\{\bniu(t)\}$ satisfy the hypotheses of Lemma~\ref{lm:int-bound}. Then, for each $\Vap>0$, there exists a random time $t_{\Vap}$, such that, $\|\mziu(t:t_{0})\|\leq\Vap$ for all $t_{\Vap}\leq t_{0}\leq t$.
\end{corollary}
\begin{proof} Note that, for each $t\geq t_{0}$,
\begin{align}
\label{corr:int-bound2}
\left\|\mziu(t:t_{0})\right\|=\left\|\mziu(t:0)-\left(\prod_{s=t_{0}}^{t-1}\left(I_{N}-\biu(t)L_{t}-\aiu(t)I_{N}\right)\right)\mziu(t_{0}:0)\right\|\\
\leq\left\|\mziu(t:0)\right\|+\left\|\mziu(t_{0}:0)\right\|,
\end{align}
where, to obtain the last inequality, we use that fact that, under~\textbf{(M.5)} (see also Remark~\ref{rem:weights}),
\begin{equation}
\label{corr:int-bound3}\left\|I_{N}-\biu(t)L_{t}-\aiu(t)I_{N}\right\|\leq 1,~~\forall t\geq 0.
\end{equation}
By Lemma~\ref{lm:int-bound}, $\mziu(t:0)\rightarrow\mathbf{0}$ as $t\rightarrow\infty$ a.s., and, hence, there exists $t_{\Vap}$, such that,
\begin{equation}
\label{corr:int-bound4}
\left\|\mziu(t:0)\right\|\leq \Vap/2,~~\forall t\geq t_{\Vap}.
\end{equation}
The result follows immediately from~\eqref{corr:int-bound3}-\eqref{corr:int-bound4}.
\end{proof}

The following order-preserving property is readily verifiable.
\begin{proposition}
\label{prop:order} Under~\textbf{(M.4)}-\textbf{(M.5)}, for each $t\geq 0$, the linear operator $\left(I_{N}-\biu(t)L_{t}-\aiu(t)I_{N}\right)$ is order-preserving on $\mathbb{R}^{N}$, i.e., for all $\mathbf{x}$ and $\mathbf{y}$ in $\mathbb{R}^{N}$ with $\mathbf{x}\leqc\mathbf{y}$, we have
\begin{equation}
\label{prop:order1}
\left(I_{N}-\biu(t)L_{t}-\aiu(t)I_{N}\right)\mathbf{x}\leqc\left(I_{N}-\biu(t)L_{t}-\aiu(t)I_{N}\right)\mathbf{y}.
\end{equation}
\end{proposition}
\begin{proof}
For the matrix $\left(I_{N}-\biu(t)L_{t}-\aiu(t)I_{N}\right)$, note that, under~\textbf{(M.5)} (see also Remark~\ref{rem:weights}), the diagonal elements are non-negative. The off-diagonal elements being negatively scaled versions of those of the Laplacian $L_{t}$ are also non-negative (by definition). Hence, the matrix $\left(I_{N}-\biu(t)L_{t}-\aiu(t)I_{N}\right)$ is non-negative, and $(\mathbf{x}-\mathbf{y})\leqc\mathbf{0}$ implies
\begin{equation}
\label{prop:order2}\left(I_{N}-\biu(t)L_{t}-\aiu(t)I_{N}\right)\left(\mathbf{x}-\mathbf{y}\right)\leqc\mathbf{0},
\end{equation}
from which the desired property follows.
\end{proof}

\section{Convergence of $\mathcal{QD}$-learning}
\label{sec:conv} The current section focuses on the convergence analysis of $\mathcal{QD}$-learning. Section~\ref{sub:bounded} obtains the boundedness of $\mathcal{QD}$-learning, whereas, Section~\ref{sub:consensus} establishes consensus of the agent updates to the networked average behavior. Finally, Section~\ref{sub:proof} completes the proof of Theorem~\ref{th:main_res} by studying the properties of the resulting averaged network dynamics.

\subsection{$\mathcal{QD}$-learning: Boundedness}
\label{sub:bounded} This section is devoted to obtaining the following boundedness of the $\mathcal{QD}$ iterates:
\begin{lemma}
\label{lm:bounded}
For each agent $n$, the successive refinement sequence $\{\mathbf{Q}^{n}_{t}\}$ is pathwise bounded, i.e.,
\begin{equation}
\label{lm:bounded1}\mathbb{P}\left(\sup_{t\geq 0}\left\|\mathbf{Q}_{t}^{n}\right\|_{\infty}<\infty\right)=1.
\end{equation}
\end{lemma}
\begin{proof} The proof is inspired by a corresponding development in~\cite{Tsitsiklis-Q} for the single-agent (centralized) case. Recall the local $\mathcal{QD}$-learning operator $\mathcal{G}^{n}(\cdot)$ defined in~\eqref{def_Giu}. By~\eqref{def_Giu3}, for each $n$ and state-action pair $(i,u)$
\begin{align}Q^{n}_{i,u}(t+1)=Q^{n}_{i,u}(t)-\beta_{i,u}(t)\sum_{l\in\Omega_{n}(t)}\left(Q^{n}_{i,u}(t)-Q^{l}_{i,u}(t)\right)\\+\alpha_{i,u}(t)\left(\mathcal{G}^{n}_{\mathbf{x}_{t},\mathbf{u}_{t}}(\mathbf{Q}^{n}_{t})-Q^{n}_{\mathbf{x}_{t},\mathbf{u}_{t}}(t)+\mathbf{\nu}^{n}_{\mathbf{x}_{t},\mathbf{u}_{t}}(t)\right),
\end{align}
Denoting by $\{\Qiu(t)\}$ the $\{\mathcal{F}_{t}\}$ adapted process with $\Qiu(t)=[Q^{1}_{i,u}(t),\cdots,Q^{N}_{i,u}(t)]^{T}$,
we note that
\begin{equation}
\label{lm:bounded3}
\Qiu(t+1)=\left(I_{N}-\biu(t)L_{t}-\aiu(t)I_{N}\right)\Qiu(t)+\aiu(t)\left(\Giu(\mathbf{Q}_{t})+\niu(t)\right),
\end{equation}
where $\Giu(\mathbf{Q}_{t})=\left[\mathcal{G}^{1}_{i,u}(\mathbf{Q}^{n}_{t}),\cdots,\mathcal{G}^{N}_{i,u}(\mathbf{Q}^{N}_{t})\right]^{T}$ and $\niu(t)$ is defined as $[\nu^{1}_{\mathbf{x}_{t},\mathbf{u}_{t}},\cdots,\nu^{N}_{\mathbf{x}_{t},\mathbf{u}_{t}}]^{T}$ (see~\eqref{def_Giu4}) on the event $\{\mathbf{x}_{t}=i,\mathbf{u}_{t}=u\}$, and is taken to be zero elsewhere. By~\eqref{def_Giu}-\eqref{def_Giu4}, it follows that
$\mathbb{E}[\niu(t)~|~\mathcal{F}_{t}]=\mathbf{0}$ for all $t$, and there exist positive constants $c_{1}$ and $c_{2}$, such that
\begin{equation}
\label{lm:bounded4}
\mathbb{E}\left[\left\|\niu(t)\right\|^{2}~|~\mathcal{F}_{t}\right]\leq c_{1}+c_{2}\left\|\mathbf{Q}_{t}\right\|^{2},
\end{equation}
with $\mathbf{Q}_{t}$ denoting the $\mathbb{R}^{N|\mathcal{X}\times\mathcal{U}|}$-valued vector collecting the $\mathbf{Q}^{n}_{t}$'s for $n=[1,\cdots,N]$. Finally, note that, for each $n$ and state-action pair $(i,u)$,
\begin{equation}
\label{lm:bounded5}
\left|\mathcal{G}^{n}_{i,u}(Q)\right|\leq c_{3}+\gamma\left\|Q\right\|_{\infty}
\end{equation}
for all $Q\in\RXU$, where $c_{3}>0$ is a constant. Thus, there exist $\hg\in [0,1)$ and a constant $J>0$, such that
\begin{equation}
\label{lm:bounded6}
\left|\mathcal{G}^{n}_{i,u}(Q)\right|\leq\hg\max\left(\|Q\|_{\infty},J\right)
\end{equation}
for all $Q\in\RXU$. Also, let $\hvap$ be such that $\hg(1+\hvap)=1$.

Now consider the $\{\mathcal{F}_{t}\}$ adapted process $\{M_{t}\}$, given by
\begin{equation}
\label{lm:bounded7} M_{t}=\max_{s\leq t}\left\|\mathbf{Q}_{t}\right\|_{\infty},~~\forall t.
\end{equation}
Let $\{J_{t}\}$ be another $\{\mathcal{F}_{t}\}$ adapted process with $J_{0}=J$, and for each $t>0$, $J_{t}=J_{t-1}$ on the event $\{M_{t}\leq (1+\hvap)J_{t-1}\}$; otherwise, i.e., if $M_{t}>(1+\hvap)J_{t-1}$, $J_{t}$ is defined by $J_{t}=J(1+\hvap)^{k}$, where $k>0$ is chosen to satisfy
\begin{equation}
\label{lm:bounded8}
J\left(1+\hvap\right)^{k-1}<M_{t}\leq J\left(1+\hvap\right)^{k}.
\end{equation}
The following hold by the above construction:
\begin{equation}
\label{lm:bounded9}M_{t}\leq \left(1+\hvap\right)J_{t},~~\forall t\geq 0,
\end{equation}
\begin{equation}
\label{lm:bounded10}M_{t}\leq J_{t}~~\mbox{if $J_{t-1}<J_{t}$}.
\end{equation}
Assume, on the contrary, that $\{\mathbf{Q}_{t}\}$ is not bounded a.s. Then, there exists an event $\mathcal{B}$ of positive measure, such that $M_{t}\rightarrow\infty$ as $t\rightarrow\infty$ on $\mathcal{B}$.

To set up a contradiction argument, consider, for each state-action pair $(i,u)$ and $t_{0}\geq 0$, the process $\{\mziu(t:t_{0})\}_{t\geq t_{0}}$ that evolves as
\begin{equation}
\label{lm:bounded11}\mziu(t+1:t_{0})=\left(I_{N}-\biu(t)L_{t}-\aiu(t)I_{N}\right)\mziu(t:t_{0})+\aiu(t)\bniu(t),
\end{equation}
in which $\mziu(t_{0}:t_{0})=\mathbf{0}$ and $\bniu(t)$ is a scaled version of $\niu(t)$ (see~\eqref{lm:bounded3}-\eqref{lm:bounded4}), such that $\bniu(t)=\niu(t)/J_{t}$. Note that $\mathbb{E}[\bniu(t)~|~\mathcal{F}_{t}]=\mathbf{0}$, which follows from~\eqref{lm:bounded4} and the fact that $J_{t}$ is adapted to $\mathcal{F}_{t}$, and
\begin{align}
\label{lm:bounded12}
\mathbb{E}\left[\left\|\bniu(t)\right\|^{2}~|~\mathcal{F}_{t}\right]=\frac{1}{J_{t}^{2}}\mathbb{E}\left[\left\|\niu(t)\right\|^{2}~|~\mathcal{F}_{t}\right]\\
\leq \frac{c_{1}}{J_{t}^{2}}+\frac{c_{2}\left\|\mathbf{Q}_{t}\right\|^{2}}{J_{t}^{2}}\leq\frac{c_{1}}{J_{0}^{2}}+\frac{c_{4}M_{t}^{2}}{J_{t}^{2}}\\
\leq  \frac{c_{1}}{J_{0}^{2}}+c_{4}\left(1+\hvap\right)^{2}\leq c_{5},
\end{align}
where $c_{4}$ and $c_{5}$ are positive constants and we use~\eqref{lm:bounded9} to obtain the penultimate inequality. Clearly, the construction~\eqref{lm:bounded11}-\eqref{lm:bounded12} falls under the purview of Corollary~\ref{corr:int-bound}, and we conclude that there exists an a.s. finite time $t_{\hvap}$, such that,
\begin{equation}
\label{lm:bounded13} \left\|\mziu(t:t_{0})\right\|\leq\hvap
\end{equation}
for all $t_{\hvap}\leq t_{0}\leq t$ and state-action pairs $(i,u)$.

The hypothesis that $M_{t}\rightarrow\infty$ on the event $\mathcal{B}$ implies, by~\eqref{lm:bounded9}, that $J_{t}\rightarrow\infty$ as $t\rightarrow\infty$ on $\mathcal{B}$. Hence, by~\eqref{lm:bounded10}, we may conclude that on $\mathcal{B}$ the inequality $M_{t}\leq J_{t}$ holds infinitely often. Together with the construction in~\eqref{lm:bounded11}-\eqref{lm:bounded13}, the above establishes the existence of an a.s. finite (random) time $t_{1}$, such that, on the event $\mathcal{B}$, $M_{t_{1}}\leq J_{t_{1}}$ and
\begin{equation}
\label{lm:bounded14}
\left\|\mziu(t:t_{1})\right\|_{\infty}\leq\hvap
\end{equation}
for all $t\geq t_{1}$ and state-action pairs $(i,u)$.

To obtain a contradiction, we now show that, under the hypothesis $M_{t}\rightarrow\infty$ as $t\rightarrow\infty$ on $\mathcal{B}$, the following set of inequalities hold a.s. on $\mathcal{B}$ for all state-action pairs $(i,u)$ and $t\geq t_{1}$:
\begin{align}
\label{lm:bounded15}
-J_{t_{1}}\left(1+\hvap\right)\mathbf{1}_{N}\leqc -J_{t_{1}}\left(\mziu(t:t_{1})+\mathbf{1}_{N}\right)\leqc\Qiu(t)\\
\leqc J_{t_{1}}\left(\mziu(t:t_{1})+\mathbf{1}_{N}\right)\leqc J_{t_{1}}\left(1+\hvap\right)\mathbf{1}_{N},
\end{align}
and
\begin{equation}
\label{lm:bounded16}
J_{t}=J_{t_{1}}.
\end{equation}
Before deriving the above, we note that~\eqref{lm:bounded15}-\eqref{lm:bounded16} would imply that
\begin{equation}
\label{lm:bounded17}
\limsup_{t\rightarrow\infty}M_{t}\leq J_{t_{1}}\left(1+\hvap\right)<\infty
\end{equation}
a.s. on $\mathcal{B}$, thus contradicting the hypothesis that $M_{t}\rightarrow\infty$ a.s. on the event $\mathcal{B}$ of positive measure. Hence, to establish Lemma~\ref{lm:bounded}, it suffices to obtain~\eqref{lm:bounded15}-\eqref{lm:bounded16} which is pursued in the following.

We proceed by induction to establish~\eqref{lm:bounded15}-\eqref{lm:bounded16}. Note that the claim holds trivially for $t=t_{1}$ as, by construction, $\mziu(t_{1}:t_{1})=\mathbf{0}$ and $\|\Qiu(t_{1})\|_{\infty}\leq M_{t_{1}}\leq J_{t_{1}}$ for all state-action pairs $(i,u)$. Assume that~\eqref{lm:bounded15}-\eqref{lm:bounded16} holds for all $s\in\{t_{1},\cdots,t\}$. To obtain~\eqref{lm:bounded15}-\eqref{lm:bounded16} for the $(t+1)$-th instant, we note that, under the induction hypothesis and by the order-preserving property in Proposition~\ref{prop:order}, we have
\begin{align}
\label{lm:bounded18}\left(I_{N}-\biu(t)L_{t}-\aiu(t)I_{N}\right)\Qiu(t)\\
\leqc\left(I_{N}-\biu(t)L_{t}-\aiu(t)I_{N}\right)\left(J_{t_{1}}\mziu(t:t_{1})+J_{t_{1}}\mathbf{1}_{N}\right)\\
=J_{t_{1}}\left(I_{N}-\biu(t)L_{t}-\aiu(t)I_{N}\right)\mziu(t:t_{1})+\left(1-\aiu(t)\right)J_{t_{1}}\mathbf{1}_{N},
\end{align}
where we also use the property of the Laplacian that $L_{t}\mathbf{1}_{N}=\mathbf{0}$. From~\eqref{lm:bounded6}, \eqref{lm:bounded11} and \eqref{lm:bounded17}, and the induction hypothesis we obtain
\begin{align}
\label{lm:bounded19}
\Qiu(t+1)=\left(I_{N}-\biu(t)L_{t}-\aiu(t)I_{N}\right)\Qiu(t)+\aiu(t)\left(\Giu(\mathbf{Q}_{t})+\niu(t)\right)\\
\leqc J_{t_{1}}\left(I_{N}-\biu(t)L_{t}-\aiu(t)I_{N}\right)\mziu(t:t_{1})+\left(1-\aiu(t)\right)J_{t_{1}}\mathbf{1}_{N}\\+\aiu(t)\left(\Giu(\mathbf{Q}_{t})+\niu(t)\right)\\
\leqc J_{t_{1}}\left(I_{N}-\biu(t)L_{t}-\aiu(t)I_{N}\right)\mziu(t:t_{1})+\left(1-\aiu(t)\right)J_{t_{1}}\mathbf{1}_{N}\\+\aiu(t)\hg(1+\hvap)J_{t_{1}}\mathbf{1}_{N}+\aiu(t)J_{t_{1}}\bniu(t)\\
= J_{t_{1}}\mathbf{1}_{N} + J_{t_{1}}\left[\left(I_{N}-\biu(t)L_{t}-\aiu(t)I_{N}\right)\mziu(t:t_{1})+\aiu(t)\bniu(t)\right]\\
=J_{t_{1}}\mathbf{1}_{N}+J_{t_{1}}\mziu(t+1:t_{1})=J_{t_{1}}\left(\mziu(t+1:t_{1})+\mathbf{1}_{N}\right),
\end{align}
which establishes the upper bound in~\eqref{lm:bounded15} at $t+1$. The lower bound can be obtained similarly by invoking the order-preserving property and the induction hypothesis in the reverse direction. Finally, to obtain~\eqref{lm:bounded16} at $t+1$, we note that the satisfaction of~\eqref{lm:bounded15} at $t+1$ implies by the induction hypothesis,
\begin{align}
\label{lm:bounded20}
M_{t+1}=\left\|\mathbf{Q}_{t}\right\|_{\infty}\leq\left(1+\hvap\right)J_{t_{1}}=\left(1+\hvap\right)J_{t},
\end{align}
and, hence, by definition, $J_{t+1}=J_{t}=J_{t_{1}}$. This establishes the desired set of inequalities~\eqref{lm:bounded15}-\eqref{lm:bounded16} for all $t\geq t_{1}$ and Lemma~\ref{lm:bounded} follows by the contradiction argument stated above~\eqref{lm:bounded17}.
\end{proof}

\subsection{$\mathcal{QD}$-learning: Asymptotic Consensus}
\label{sub:consensus} In this section, we establish the asymptotic agreement in agent updates. Recall, for each $n$, $\{\mathbf{Q}^{n}_{t}\}$ to be the $\{\mathcal{F}_{t}\}$ adapted update sequence at agent $n$ (see~\eqref{up_Q}). Denote by $\{\bQ_{t}\}$ the network-averaged iterate process, i.e.,
\begin{equation}
\label{sub:cons1}\bQ_{t}=\left(1/N\right)\sum_{n=1}^{N}\mathbf{Q}^{n}_{t},~~\forall t.
\end{equation}
The goal of the section is to show that the local agent iterates eventually merge to the network-averaged behavior. Specifically, we will establish the following:
\begin{lemma}
\label{lm:cons} The agents reach consensus asymptotically, i.e., for each $n$,
\begin{equation}
\label{lm:cons1}\mathbb{P}\left(\lim_{t\rightarrow\infty}\left\|\mathbf{Q}^{n}_{t}-\bQ_{t}\right\|=0\right)=1.
\end{equation}
\end{lemma}
\begin{proof} Recall, by~\eqref{lm:bounded3}, for each state-action pair $(i,u)$, the process $\{\Qiu(t)\}$ evolves as
\begin{equation}
\label{lm:cons2}
\Qiu(t+1)=\left(I_{N}-\biu(t)L_{t}-\aiu(t)I_{N}\right)\Qiu(t)+\aiu(t)\left(\Giu(\mathbf{Q}_{t})+\niu(t)\right),
\end{equation}
which, by~\eqref{up_Q}, may be rewritten as
\begin{equation}
\label{lm:cons3}
\Qiu(t+1)=\left(I_{N}-\biu(t)L_{t}-\aiu(t)I_{N}\right)\Qiu(t)+\aiu(t)\left(\mathbf{U}(t)+\mathbf{J}(t)\right),
\end{equation}
where $\{\mathbf{U}(t)\}$ and $\{\mathbf{J}(t)\}$ are $\mathbb{R}^{N}$-valued processes whose $n$-th components are given by
\begin{equation}
\label{lm:cons4}\mathbf{U}_{n}(t)=\gamma\min_{v\in\mathcal{U}}Q^{n}_{\mathbf{x}_{t+1},v}(t)~~\mbox{and}~~\mathbf{J}_{n}(t)=c_{n}(\mathbf{x}_{t},\mathbf{u}_{t}),
\end{equation}
respectively.

For each $k\geq 0$ denote by $\mathcal{H}_{k}$ the $\sigma$-algebra associated with the stopping time $T_{i,u}(k)$, see~\eqref{def_Tiu}, i.e., $\mathcal{H}_{k}=\mathcal{F}_{T_{i,u}(k)}$. By $\{\mathbf{z}_{k}\}$ denote the randomly sampled version of $\{\Qiu(t)\}$, i.e., for each $k$
\begin{equation}
\label{lm:cons5}\mathbf{z}_{k}=\Qiu(T_{i,u}(k)),
\end{equation}
and note that the process $\{\mathbf{z}_{k}\}$ is $\{\mathcal{H}_{k}\}$ adapted. Noting that the process $\{\Qiu(t)\}$ may only change at the stopping time $T_{i,u}(k)$'s, the process $\{\mathbf{z}_{k}\}$ evolves as
\begin{equation}
\label{lm:cons6}
\mathbf{z}_{k+1}=\left(I_{N}-\beta_{k}\hL_{k}-\alpha_{k}I_{N}\right)\mathbf{z}_{k}+\alpha_{k}\left(\hmu(k)+\hmj(k)\right),
\end{equation}
where, by~\textbf{(M.4)}, we have
\begin{equation}
\label{lm:cons7}
\beta_{k}\doteq\biu(T_{i,u}(k))=b/(k+1)^{\tau_{2}},
\end{equation}
\begin{equation}
\label{lm:cons8}
\alpha_{k}\doteq\aiu(T_{i,u}(k))=a/(k+1)^{\tau_{1}}
\end{equation}
for all $k\geq 0$. Finally, denoting by $L_{k}$ and $\hmj(k)$, the quantities $L_{T_{i,u}(k)}$ and $\mathbf{J}(T_{i,u}(k))$ respectively, by~\textbf{(M.1)}-\textbf{(M.2)} we conclude that the processes $\{L_{k}\}$ and $\{\hmj(k)\}$ are $\{\mathcal{H}_{k+1}\}$ adapted with $L_{k}$ and $\hmj(k)$ being independent of $\mathcal{H}_{k}$ for each $k$. Further, for each $k$, $\mathbb{E}[L_{k}|\mathcal{H}_{k}]=\OL$ and the i.i.d. process $\{\hmj(k)\}$ satisfies the moment condition
\begin{equation}
\label{lm:cons9}\mathbb{E}\left[\left\|\hmj(k)\right\|^{2+\Vap_{1}}\right]<\infty,
\end{equation}
for a constant $\Vap_{1}>0$ (see~\eqref{M1moment}).

Let $\oz_{k}=(1/N)\mathbf{1}_{N}^{T}\mathbf{z}_{k}$ denote the average of the components of $\mathbf{z}_{k}$. Using standard properties of the Laplacian $L_{k}$ and~\eqref{lm:cons6}, it follows that the residual $\wmz_{k}=\mathbf{z}_{k}-\oz_{k}\mathbf{1}_{N}$ evolves as
\begin{equation}
\label{lm:cons10}
\wmz_{k+1}=\left(I_{N}-\beta_{k}L_{k}-\alpha_{k}I_{N}\right)\wmz_{k}+\alpha_{k}\left(\wmu_{k}+\wmj_{k}\right),
\end{equation}
where
\begin{equation}
\label{lm:cons11}
\wmu_{k}=\left(I-(1/N)\mathbf{1}_{N}\mathbf{1}_{N}^{T}\right)\hmu(k)~~\mbox{and}~~\wmj_{k}=\left(I-(1/N)\mathbf{1}_{N}\mathbf{1}_{N}^{T}\right)\hmj(k)
\end{equation}
for all $k$. Noting that, by construction, $\wmz_{k}\in\PC$ (see Definition~\ref{def:cons}) for all $k$, and hence, by Lemma~\ref{lm:conn}, there exists a measurable $\{\mathcal{H}_{k+1}\}$ adapted $\mathbb{R}_{+}$ valued process $\{r_{k}\}$ and a constant $c_{r}>0$, such that $0\leq r_{k}\leq 1$ a.s. and
\begin{align}
\label{lm:cons12}
\left\|\left(I_{N}-\beta_{k}L_{k}-\alpha_{k}I_{N}\right)\wmz_{k}\right\|\leq\left\|\left(I_{N}-\beta_{k}L_{k}\right)\wmz_{k}\right\|+\alpha_{k}\wmz_{k}\\
\leq \left(1-r_{k}\right)\left\|\wmz_{k}\right\|+\alpha_{k}\left\|\wmz_{k}\right\|
\end{align}
with
\begin{equation}
\label{lm:cons120}
\mathbb{E}\left[r_{k}~|~\mathcal{H}_{k}\right]\geq\frac{c_{r}}{(k+1)^{\tau_{2}}}
\end{equation}
for all $k$. Since $\tau_{2}<\tau_{1}$ (see Assumption~\textbf{(M.5)}), there exists $k_{0}$ (deterministic) and another constant $c_{2}\in (0,1)$, such that,
\begin{equation}
\label{lm:cons13}
\left\|\left(I_{N}-\beta_{k}L_{k}-\alpha_{k}I_{N}\right)\wmz_{k}\right\|\leq\left(1-c_{2}r_{k}\right)\left\|\wmz_{k}\right\|
\end{equation}
for $k\geq k_{0}$. By~\eqref{lm:cons10} and~\eqref{lm:cons13} we obtain for all $k\geq k_{0}$
\begin{equation}
\label{lm:cons14}
\left\|\wmz_{k+1}\right\|\leq\left(1-c_{2}r_{k}\right)\left\|\wmz_{k}\right\|+\alpha_{k}\left(\left\|\wmu_{k}\right\|+\left\|\wmj_{k}\right\|\right).
\end{equation}
Note that the process $\{\|\wmu_{k}\|\}$ is pathwise bounded by Lemma~\ref{lm:bounded} and $\{\|\wmj_{k}\|\}$ is i.i.d. satisfying the moment condition in~\eqref{lm:cons9}. Hence, the update in~\eqref{lm:cons14} falls under the purview of Lemma~\ref{lm:mean-conv}, and we conclude that
$(k+1)^{\tau}\wmz_{k}\rightarrow\mathbf{0}$ as $k\rightarrow\infty$ a.s. for all $\tau\in\left(0,\tau_{1}-\tau_{2}-1/(2+\Vap_{1})\right)$. In particular, $\wmz_{k}\rightarrow\mathbf{0}$ as $k\rightarrow\infty$ a.s., and, since $\{\Qiu(t)\}$ is a piecewise constant interpolation of $\{\mathbf{z}_{k}\}$, we obtain
\begin{equation}
\label{lm:cons15}
\mathbb{P}\left(\lim_{t\rightarrow\infty}\left|Q^{n}_{i,u}(t)-\overline{Q}_{i,u}(t)\right|=0\right)=1,
\end{equation}
for each agent $n$, with $\overline{Q}_{i,u}(t)=(1/N)\mathbf{1}_{N}^{T}\Qiu(t)$ denoting the component-wise average of $\Qiu(t)$. Since the above can be shown for each state-action pair $(i,u)$, the assertion follows.
\end{proof}

\subsection{$\mathcal{QD}$-learning: Averaged Dynamics}
\label{sub:proof} This section investigates the asymptotics of the network-averaged iterate $\{\bQ_{t}\}$ (see~\eqref{sub:cons1}). Since the agents reach consensus asymptotically (Lemma~\ref{lm:cons}), it suffices to establish the convergence of $\{\bQ_{t}\}$ in order to obtain the main result of this paper.

To this end, consider the (centralized) $Q$-learning operator $\bG:\RXU\mapsto\RXU$, whose $(i,u)$-th component $\bGiu:\RXU\mapsto\mathbb{R}$ is given by
\begin{equation}
\label{sub:proof1}
\bGiu\left(\mathbf{Q}\right)=(1/N)\sum_{n=1}^{N}\mathbb{E}\left[c_{n}(i,u)\right]+\gamma\sum_{j\in\mathcal{X}}p_{i,j}^{u}\min_{v\in\mathcal{U}}Q_{j,v}
\end{equation}
for all $\mathbf{Q}\in\mathbb{R}^{|\mathcal{X}\times\mathcal{U}|}$. Note that, informally, $\bG(\cdot)$ is the average of the local $Q$-learning operators, i.e., for each $\mathbf{Q}\in\RXU$ and state-action pair $(i,u)$, we have
\begin{equation}
\label{sub:proof2}\bGiu(\mathbf{Q})=(1/N)\sum_{n=1}^{N}\mathcal{G}^{n}_{i,u}(\mathbf{Q}),
\end{equation}
where the operators $\mathcal{G}^{n}_{i,u}(\cdot)$ are defined in~\eqref{def_Giu}. It is readily seen that the following assertion holds:
\begin{proposition}
\label{prop:G} The (centralized) $Q$-learning operator is contraction. Specifically, we have
\begin{equation}
\label{prop:G1}\left\|\bG\left(\mathbf{Q}\right)-\bG\left(\mathbf{Q}^{\prime}\right)\right\|_{\infty}\leq\gamma\left\|\mathbf{Q}-\mathbf{Q}^{\prime}\right\|_{\infty}~~\forall\mathbf{Q},\mathbf{Q}^{\prime}\in\RXU.
\end{equation}
Also, denoting by $\mathbf{Q}^{\ast}$ the unique fixed point of $\bG(\cdot)$, we note that $\min_{u\in\mathcal{U}}Q^{\ast}_{i,u}=V^{\ast}_{i}$ for each $i\in\mathcal{X}$, where $V^{\ast}_{i}$ denotes the $i$-th component of the unique fixed point $\mathbf{V}^{\ast}$ of the (centralized) dynamic programming operator $\mathcal{T}(\cdot)$, see~\eqref{opt-cost}-\eqref{def_T}.
\end{proposition}

The convergence of the network-averaged iterate process $\{\bQ_{t}\}$ (see~\eqref{sub:cons1}) will be established in this section as follows:
\begin{lemma}
\label{lm:conv} Under~\textbf{(M.1)-(M.5)} we have
\begin{equation}
\label{lm:conv1}
\mathbb{P}\left(\lim_{t\rightarrow\infty}\left\|\bQ_{t}-\mathbf{Q}^{\ast}\right\|=0\right)=1,
\end{equation}
where $\mathbf{Q}^{\ast}$ is the fixed point of $\mathcal{G}(\cdot)$ (see Proposition~\ref{prop:G}).
\end{lemma}

The following result will be used in the proof of Lemma~\ref{lm:conv}.
\begin{lemma}
\label{lm:intcon}
For each state-action pair $(i,u)$, let $\{\sziu(t)\}$ denote the $\{\mathcal{F}_{t}\}$ adapted real-valued process evolving as
\begin{equation}
\label{lm:intcon1}
\mziu(t+1)=\left(1-\alpha_{i,u}(t)\right)\sziu(t)+\alpha_{i,u}(t)\left(\bniu(t)+\bviu(t)\right),
\end{equation}
where the weight sequence $\{\alpha_{i,u}(t)\}$ is given by~\eqref{def_alpha} and $\{\bniu(t)\}$ is an $\{\mathcal{F}_{t+1}\}$ adapted process satisfying $\mathbb{E}[\bniu(t)~|~\mathcal{F}_{t}]=0$ for all $t$ and
\begin{equation}
\label{lm:intcon2}
\mathbb{P}\left(\sup_{t\geq 0}\mathbb{E}\left[\bniu^{2}(t)~|~\mathcal{F}_{t}\right]<\infty\right)=1.
\end{equation}
Further, the process $\{\bviu(t)\}$ is $\{\mathcal{F}_{t}\}$ adapted, such that, $\bviu(t)\rightarrow 0$ as $t\rightarrow\infty$ a.s.
Then, under~\textbf{(M.4)}-\textbf{(M.5)}, we have $\sziu(t)\rightarrow 0$ as $t\rightarrow\infty$ a.s.
\end{lemma}
\begin{proof} Consider the auxiliary process $\{y_{i,u}(t)\}$ with
\begin{equation}
\label{lm:intcon3}
y_{i,u}(t+1)=\left(1-\aiu(t)\right)y_{i,u}(t)+\aiu(t)\bviu(t)
\end{equation}
for all $t$. Note that $\bviu(t)\rightarrow 0$ as $t\rightarrow\infty$ a.s. and, hence, for every $\delta>0$, there exists $t_{\delta}$ (random), such that, $\bviu(t)\leq\delta$ for all $t\geq t_{\delta}$. Hence, by~\eqref{lm:intcon3} for $t\geq t_{\delta}$, we obtain
\begin{equation}
\label{lm:intcon4}|y_{i,u}(t+1)|\leq\left(1-\aiu(t)\right)|y_{i,u}(t)|+\aiu(t)\delta,
\end{equation}
which, by a pathwise application of Proposition~\ref{lm:decsup}, leads to
\begin{equation}
\label{lm:intcon5}
\limsup_{t\rightarrow\infty}|y_{i,u}(t)|\leq\delta.
\end{equation}
Since $\delta>0$ is arbitrary in~\eqref{lm:intcon5}, we conclude that $y_{i,u}(t)\rightarrow 0$ as $t\rightarrow\infty$ a.s. Now, let us denote by $\{\wziu(t)\}$ the $\{\mathcal{F}_{t}\}$ adapted process that satisfies $\wziu(t)=\sziu(t)-y_{i,u}(t)$ for all $t$. Then,
\begin{equation}
\label{lm:intcon6}
\wziu(t+1)=\left(1-\aiu(t)\right)\wziu(t)+\aiu(t)\bniu(t)
\end{equation}
for all $t$. The hypothesis~\eqref{lm:intcon2} and Egorov's theorem implies that, for every $\pdelta>0$, there exists a constant $K_{\pdelta}>0$, such that
\begin{equation}
\label{lm:intcon7}
\mathbb{P}\left(\sup_{t\geq 0}\mathbb{E}\left[\bniu^{2}(t)~|~\mathcal{F}_{t}\right]\leq K_{\pdelta}\right)>1-\pdelta.
\end{equation}
Let $\tau_{\pdelta}$ denote
\begin{equation}
\label{lm:intcon8}
\tau_{\pdelta}=\min\left\{t\geq 0~:~\mathbb{E}\left[\bniu^{2}(t)~|~\mathcal{F}_{t}\right]>\delta\right\},
\end{equation}
and note that it readily follows that $\tau_{\pdelta}$ is a stopping time w.r.t. $\{\mathcal{F}_{t}\}$ (see~\cite{Jacod-Shiryaev}). Also, let $\{\bniu^{\pdelta}(t)\}$ be the $\{\mathcal{F}_{t+1}\}$ adapted process, such that, $\bniu^{\pdelta}(t)=\bniu(t)\mathbb{I}(t<\tau_{\pdelta})$ for all $t$, and note that
\begin{equation}
\label{lm:intcon9}
\mathbb{E}\left[\bniu^{\pdelta}(t)~|~\mathcal{F}_{t}\right]=\mathbb{I}(t<\tau_{\pdelta})\mathbb{E}\left[\bniu(t)~|~\mathcal{F}_{t}\right]=0,
\end{equation}
and
\begin{equation}
\label{lm:intcon10}\mathbb{E}\left[\left(\bniu^{\pdelta}(t)\right)^{2}~|~\mathcal{F}_{t}\right]=\mathbb{I}(t<\tau_{\pdelta})\mathbb{E}\left[\bniu^{2}(t)~|~\mathcal{F}_{t}\right]\leq K_{\pdelta},
\end{equation}
for all $t$, where the last inequality uses the definition of $\tau_{\pdelta}$, see~\eqref{lm:intcon8}. Finally, introduce the $\{\mathcal{F}_{t}\}$ adapted process $\{\wziu^{\pdelta}(t)\}$ that evolves as
\begin{equation}
\label{lm:intcon11}
\wziu^{\pdelta}(t+1)=\left(1-\aiu(t)\right)\sziu^{\pdelta}(t)+\aiu(t)\bniu^{\pdelta}(t)
\end{equation}
for all $t$, and note that by~\eqref{lm:intcon7}, we have
\begin{equation}
\label{lm:intcon12}
\mathbb{P}\left(\sup_{t\geq 0}\left|\wziu(t)-\wziu^{\pdelta}(t)\right|=0\right)>1-\pdelta.
\end{equation}
With~\eqref{lm:intcon9}-\eqref{lm:intcon10} we note that the process $\{\wziu^{\pdelta}(t)\}$ reduces to a scalar instantiation of the process in the hypothesis of Lemma~\ref{lm:int-bound} (with $\biu(t)$ set to zero for all $t$, see also Remark~\ref{rem:lm:int}), and we obtain $\wziu^{\pdelta}(t)\rightarrow 0$ as $t\rightarrow\infty$ a.s. Hence, by~\eqref{lm:intcon12} we have
\begin{equation}
\label{lm:intcon13}
\mathbb{P}\left(\lim_{t\rightarrow\infty}\wziu(t)=0\right)>1-\pdelta.
\end{equation}
Noting that $\pdelta>0$ above is arbitrary, we obtain $\sziu(t)\rightarrow 0$ as $t\rightarrow\infty$ a.s., which together with the fact that $y_{i,u}(t)\rightarrow 0$ as $t\rightarrow\infty$ a.s. yield $\sziu(t)\rightarrow 0$ as $t\rightarrow\infty$ a.s.
\end{proof}

We now complete the proof of Lemma~\ref{lm:conv}.
\begin{proof}[Lemma~\ref{lm:conv}] Noting that $\mathbf{1}_{N}^{T}L_{t}=\mathbf{0}$ and by~\eqref{lm:cons3} and~\eqref{sub:proof2}, we have for each state-action pair $(i,u)$
\begin{equation}
\label{lm:conv10}\bQiu(t+1)=\left(1-\aiu(t)\right)\bQiu(t)+\aiu(t)\left(\bGiu(\bQ_{t})+\bniu(t)+\bviu(t)\right),
\end{equation}
where $\{\bniu(t)\}$ and $\{\bviu(t)\}$ are $\{\mathcal{F}_{t+1}\}$ and $\{\mathcal{F}_{t}\}$ adapted processes, respectively, such that $\bniu(t)=(1/N)\mathbf{1}^{T}_{N}\niu(t)$ and
\begin{equation}
\label{lm:conv11}
\bviu(t)=(1/N)\sum_{n=1}^{N}\left(\mathcal{G}^{n}_{i,u}(\mathbf{Q}^{n}_{t})-\mathcal{G}^{n}_{i,u}(\bQ_{t})\right)
\end{equation}
for all $t$. Note that $\mathbb{E}[\bniu(t)|\mathcal{F}_{t}]=0$ and the boundedness of the iterate process $\{\mathbf{Q}_{t}\}$ (Lemma~\ref{lm:bounded}) implies
\begin{equation}
\label{lm:conv12}
\mathbb{P}\left(\sup_{t\geq 0}\mathbb{E}\left[\bniu^{2}(t)~|~\mathcal{F}_{t}\right]<\infty\right)=1.
\end{equation}
Observing that the functionals $\mathcal{G}^{n}_{i,u}(\cdot)$ are Lipschitz, we have
\begin{equation}
\label{lm:conv13}
|\bviu(t)|\leq c_{1}\sum_{n=1}^{N}\left\|\mathbf{Q}^{n}_{t}-\bQ_{t}\right\|
\end{equation}
for all $t$, and, hence, by Lemma~\ref{lm:cons}, we conclude that $\bviu(t)\rightarrow 0$ as $t\rightarrow\infty$ a.s.

Now consider the auxiliary $\{\mathcal{F}_{t}\}$ adapted process $\{\sziu(t)\}$ for each state-action pair $(i,u)$, such that
\begin{equation}
\label{lm:conv14}
\sziu(t+1)=\left(1-\aiu(t)\right)\sziu(t)+\aiu(t)\left(\bniu(t)+\bviu(t)\right)
\end{equation}
for all $t$. Based on the above discussion on the properties of the processes $\{\bniu(t)\}$ and $\{\bviu(t)\}$ and Lemma~\ref{lm:intcon}, we conclude that the process $\{\sziu(t)\}$, so constructed, satisfies $\sziu(t)\rightarrow 0$ as $t\rightarrow\infty$ a.s.

By Lemma~\ref{lm:bounded} the process $\{\bQ_{t}\}$ is bounded and hence there exists an a.s. finite random variable $R$, such that
\begin{equation}
\label{lm:conv15} R=\limsup_{t\rightarrow\infty}\left\|\bQ_{t}-\mathbf{Q}^{\ast}\right\|_{\infty}.
\end{equation}
Assume on the contrary that $R\neq 0$ a.s. Then there exists an event $\mathcal{B}$ of positive measure such that $R>0$ on $\mathcal{B}$. To derive a contradiction, let $\delta>0$ be a constant, such that, $\gamma(1+\delta)<1$ and consider the process $\{\hQiu(t)\}$, for each state-action pair $(i,u)$, such that $\hQiu(t)=\bQiu(t)-\sziu(t)-Q^{\ast}_{i,u}$ for all $t$. Noting that $\mathbf{Q}^{\ast}$ is a fixed point of the operator $\bG(\cdot)$, we have using~\eqref{lm:conv10} and~\eqref{lm:conv14}
\begin{equation}
\label{lm:conv16}
\hQiu(t+1)=\left(1-\aiu(t)\right)\hQiu(t)+\aiu(t)\left(\bGiu(bQ_{t})-\bGiu(\mathbf{Q}^{\ast})\right)
\end{equation}
for all $t$. Hence, there exists $t_{\delta}$ (random), such that,
\begin{equation}
\label{lm:conv17}\left\|\bQ_{t}-\mathbf{Q}^{\ast}\right\|_{\infty}\leq R(1+\delta)
\end{equation}
on $\mathcal{B}$ a.s. for all $t\geq t_{\delta}$, Thus, by~\eqref{lm:conv16},
\begin{equation}
\label{lm:conv18}
|\hQiu(t+1)|\leq\left(1-\aiu(t)\right)|\hQiu(t)|+\aiu(t)\gamma(1+\delta)R
\end{equation}
on $\mathcal{B}$ a.s. for all $t\geq t_{\delta}$, where we use the fact that, for each $(i,u)$, the functional $\bGiu(\cdot)$ is a contraction with coefficient $\gamma$, see~\eqref{prop:G1}. A pathwise application of Proposition~\ref{lm:decsup} on~\eqref{lm:conv18} then yields
\begin{equation}
\label{lm:conv19}
\mathbb{P}\left(\limsup_{t\rightarrow\infty}\left|\hQiu(t)\right|\leq\gamma(1+\delta)R\right)\geq\mathbb{P}\left(\mathcal{B}\right)>0.
\end{equation}
Since, the above holds for each state-action pair $(i,u)$ and $\gamma(1+\delta)<1$, we conclude that,
\begin{equation}
\label{lm:conv20}
\limsup_{t\rightarrow\infty}\left\|\bQ_{t}-\mathbf{Q}^{\ast}\right\|_{\infty}<R~~~\mbox{a.s. on $\mathcal{B}$}.
\end{equation}
Since, $\mathcal{B}$ has positive measure, the above contradicts with the hypothesis~\eqref{lm:conv15} and we conclude that $R=0$ a.s. This completes the proof.
\end{proof}

\textbf{Proof of Theorem~\ref{th:main_res}}: The first part of Theorem~\ref{th:main_res} follows from the fact that, for each $n$, $\mathbf{Q}^{n}_{t}\rightarrow\mathbf{Q}^{\ast}$ a.s. as $t\rightarrow\infty$ by Lemma~\ref{lm:cons} and Lemma~\ref{lm:conv}. The second part is an immediate consequence of the characterization of the limiting consensus value $\mathbf{Q}^{\ast}$ achieved in Proposition~\ref{prop:G}.

\section{Simulation Studies}
\label{sec:sim} In this section we simulate the convergence rate behavior of the proposed $\mathcal{QD}$-learning for an example setup. The setup consists of a network of $N=40$ agents with binary-valued state and action spaces, i.e., the cardinality of the state-action space $\mathcal{X}\times\mathcal{U}$ is 4. Thus, in all there are 8 controlled transition parameters $p_{i,j}^{u}$, $i,j\in\mathcal{X}$ and $u\in\mathcal{U}$; 4 of these probabilities were chosen independently by uniformly sampling the interval $[0,1]$, which also fixes the values of the remaining 4. For each $n$ and state-action pair $(i,u)$, the random one stage cost $c_{n}(i,u)$ is assumed to follow a Gaussian distribution with variance 40, the mean (or expected one-stage cost) $\mathbb{E}[c_{n}(i,u)]$ being another random sample (generated independently for each $n$ and state-action pair $(i,u)$) from the uniform distribution on $[0,400]$. The discounting factor was taken to be $\gamma=0.7$. For the purpose of simulating the distributed $\mathcal{QD}$ scheme, we considered a 2-nearest neighbor inter-agent communication topology with a $0.5$ probability of link erasure (for all network links), i.e., the resulting communication network may be viewed as the $N$ agents being symmetrically placed on a circle with each agent exchanging information with its two neighbors on either side. The performance of both the distributed $\mathcal{QD}$-learning and centralized $Q$-learning were simulated on a single (random) state-action trajectory $\{\mathbf{x}_{t},\mathbf{u}_{t}\}$ instantiation - specifically, the state-action trajectory was generated by independently uniformly sampling control actions from $\mathcal{U}$ over time, whereas, the state trajectory $\{\mathbf{x}_{t}\}$ instantiation was generated by sampling $\mathbf{x}_{t+1}$ from the probability distribution $p_{\mathbf{x}_{t},\cdot}^{\mathbf{u}_{t}}$ independently of the past at each time $t$. Fig.~\ref{fig:comp} illustrates the typical sample path convergence behavior of the distributed $\mathcal{QD}$-scheme and the centralized $Q$-learning, in which the distributed $Q$-factors, $Q^{n}_{i,u}(t)$, corresponding to $\mathcal{QD}$-learning were obtained using the recursions~\eqref{up_Q}-\eqref{def_alpha}, whereas, the centralized $Q$-factors, $Q^{c}_{i,u}(t)$, were generated using the centralized $Q$-learning recursions
\begin{equation}
\label{cent-Q-rec}
Q^{c}_{i,u}(t+1)=Q^{c}_{i,u}(t)+\alpha_{i,u}(t)\left((1/N)\sum_{n=1}^{N}c_{n}(\mathbf{x}_{t},\mathbf{u}_{t})+\gamma\min_{v\in\mathcal{U}}Q^{c}_{\mathbf{x}_{t+1},v}(t)-Q^{c}_{i,u}(t)\right)
\end{equation}
\begin{figure}[ptb]
\begin{center}
\includegraphics[height=2.5in, width=3in]{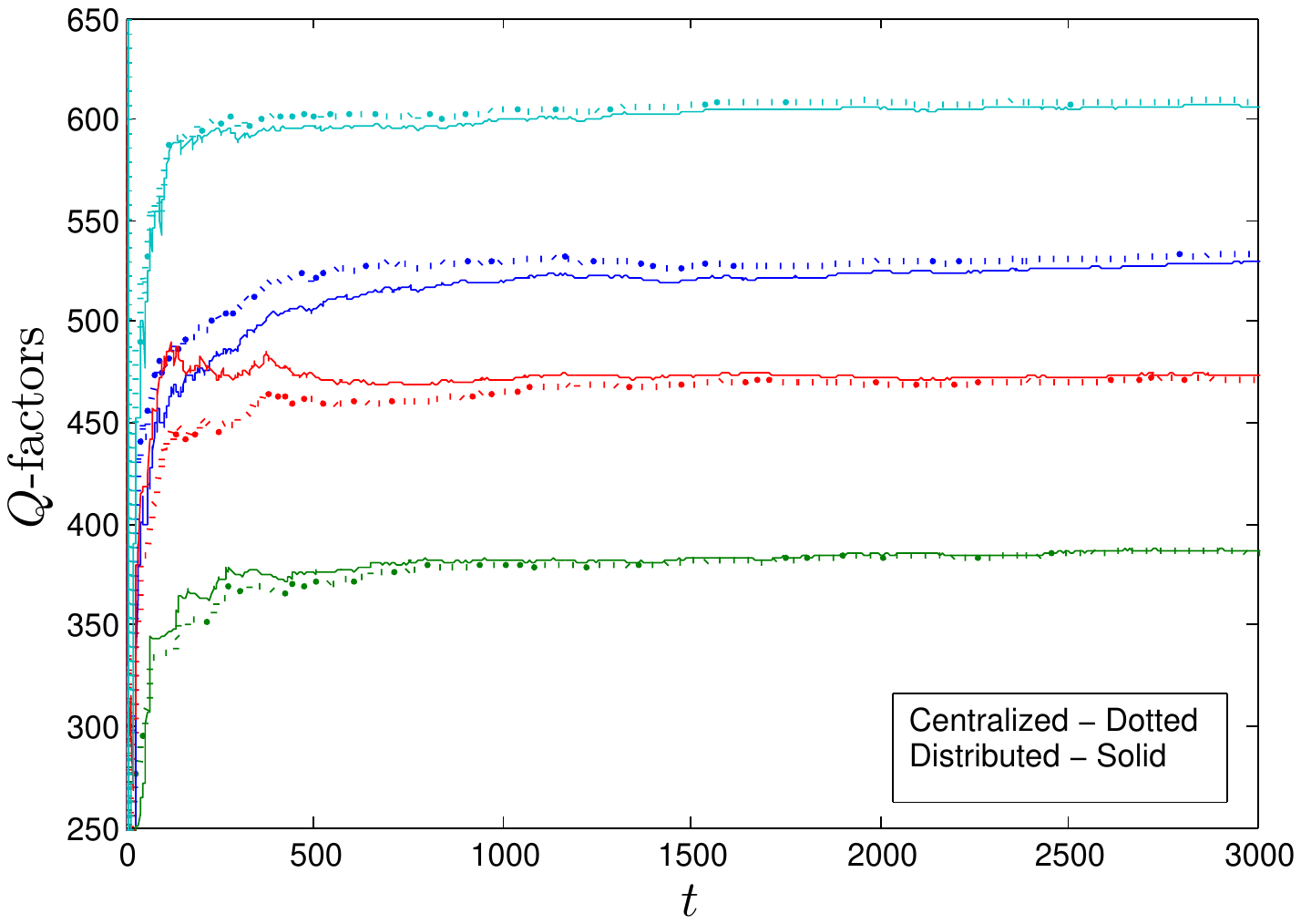}
\includegraphics[height=2.5in, width=3in]{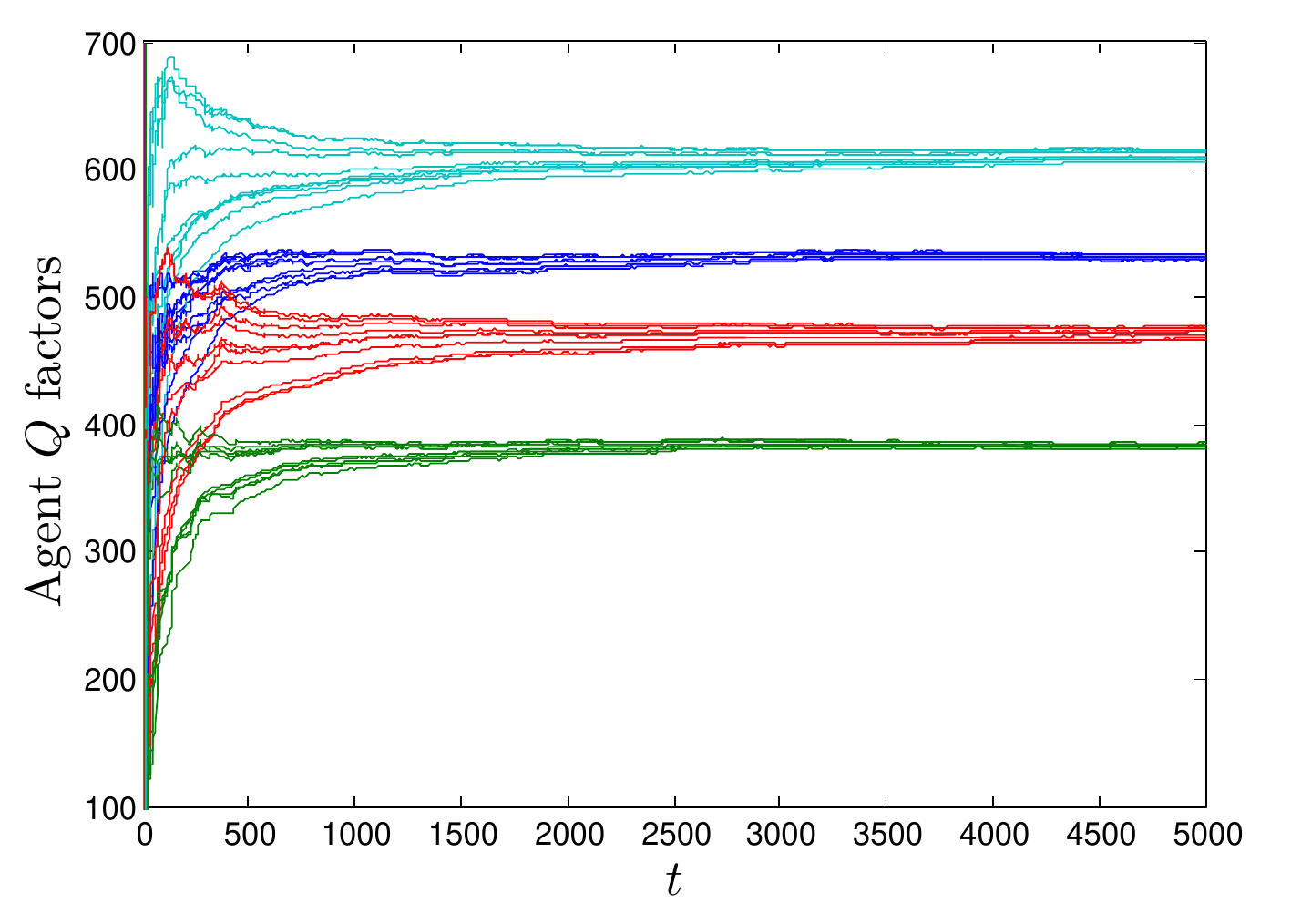}
\caption{Left: Centralized (dotted lines) and distributed $\mathcal{QD}$ (at a randomly uniformly selected agent, solid lines) $Q$-factors. Right: Consensus among distributed $Q$-factors.}
\label{fig:comp}
\label{fig:dist}
\end{center}
\end{figure}
for each state-action pair $(i,u)\in\mathcal{X}\times\mathcal{U}$. The exponent $\tau_{1}$ (see~\textbf{(M.5)}) corresponding to the consensus weight sequence $\{\beta_{i,u}(t)\}$ (in the distributed $\mathcal{QD}$) was set to $0.2$, whereas, the innovation weight sequence $\{\alpha_{i,u}(t)\}$ (for both the $\mathcal{QD}$ and centralized $Q$-learning~\eqref{cent-Q-rec}) exponent $\tau_{1}$ was taken to be 1. In Fig.~\ref{fig:comp} (on the left) we compare the evolution of the centralized $Q$ factors with that of the distributed generated by the $\mathcal{QD}$ recursions at a randomly selected agent, i.e., for each pair $(i,u)\in\mathcal{X}\times\mathcal{U}$ (with a total of 4 such pairs), we plot the trajectories $\{Q^{c}_{i,u}(t)\}$ (depicted by dotted lines) and $\{Q^{n}_{i,u}(t)\}$, at a randomly uniformly selected agent $n$ (depicted by solid lines), corresponding to the centralized and distributed respectively; whereas, Fig.~\ref{fig:dist} (on the right) illustrates the evolution of the $Q$-factors at $10$ randomly (uniformly) selected network agents (for the distributed $\mathcal{QD}$), verifying that they reach consensus on each state-action pair $(i,u)$.

From Fig.~\ref{fig:comp} we readily infer that the convergence rate of $\mathcal{QD}$ is reasonably close to that of centralized $Q$-learning - more importantly, Fig.~\ref{fig:comp} demonstrates that the per-step convergence factor (i.e., the improvement over successive time steps) of $\mathcal{QD}$ approach that of the centralized $Q$-learning asymptotically (in the limit of large time $t$). While the above inference is drawn from the low-dimensional (4 state-action pairs) simulation setup considered (and while the absolute convergence rates of both the centralized and distributed will decline with increasing dimension of the state-action space), we expect the same relative convergence rate trend (i.e., the asymptotic equivalence of the distributed and centralized per-step convergence factors) to hold for more general (higher dimensional) setups. Intuitively, the negligible asymptotic convergence rate loss of the distributed with respect to the centralized is attributed to the asymptotic domination of the consensus potential over the innovations (recall $\beta_{i,u}(t)/\alpha_{i,u}(t)\rightarrow\infty$ as $t\rightarrow\infty$), which enables each agent to essentially track the \emph{network aggregated innovation} instantaneously in the limit of large time.

\section{Conclusion}
\label{conclusion} The paper has investigated a distributed multi-agent reinforcement learning setup in a networked environment, in which the agents (for instance, temperature sensors in smart thermostatically controlled building applications, or, more generally, autonomous entities in social computing and decision making applications) respond differently to a global environmental signal or trend. Our setup is collaborative and non-competitive, with the overall network objective being global welfare, i.e., specifically, the network is interested in learning and evaluating the optimal stationary control strategy that minimizes the network-average infinite horizon discounted one-stage costs. Rather than considering a centralized solution methodology that requires each network agent to forward its instantaneous (random) one-stage cost to a remote centralized supervisor at all times, we have focused on a distributed approach in which the network agents engage in in-network processing (learning) by means of local communication and computation. The resulting distributed version of $Q$-learning, the $\mathcal{QD}$ scheme, has been shown to achieve optimal learning performance asymptotically, i.e., the network agents reach consensus on the desired value function and the corresponding optimal control strategy, under minimal connectivity assumptions on the underlying communication graph. Similar to direct adaptive control formulations (see, for example,~\cite{Tsitsiklis-Q}), we have allowed generic statistical dependence on the state-action trajectories (processes) that drive the learning, which, in turn, in our distributed setting leads to mixed time-scale stochastic evolutions that are non-Markovian (see~\eqref{up_Q} and note that the state $\mathbf{x}_{t}$ and control $\mathbf{u}_{t}$ are general $\{\mathcal{F}_{t}\}$ processes). The analysis methods developed in the paper are of independent interest and we expect our techniques to be applicable to broader classes of distributed information processing and control problems with memory. For a low dimensional (in the size of the state-action space) example, the simulations in Section~\ref{sec:sim} indicate that the convergence rate of the proposed distributed $\mathcal{QD}$ scheme is reasonably close to that of the centralized implementation. While, in the same section it was argued that the per-step convergence rate of the distributed scheme should asymptotically approach that of the centralized in more general scenarios (higher dimensional setups) due to the asymptotic domination of the consensus potential over the innovations, an important future direction would consist of analytically characterizing the convergence rate of $\mathcal{QD}$-learning under further assumptions on the state-action generation, for instance, by imposing specific statistical structure on the simulated state-action pairs, a commonly used approach being simulating the system response by i.i.d. generation of state-action pairs~\cite{Szepesvari-Q}. In such cases, or more generally, cases in which the convergence rate of centralized $Q$-learning
may be characterized~\cite{Dar-Q}, it would be interesting to see whether the proposed distributed $\mathcal{QD}$-learning entails any loss of performance (with respect to convergence rate) or not. Two other practically motivating and challenging future research topics concern the partial state information case, in which the global state process may not be perfectly observable at the local agent level, and the distributed actuation case, in which, instead of a remote controller acting on the global signal, the agents are themselves responsible for local actuations.

\bibliographystyle{IEEEtran}
\bibliography{IEEEabrv,CentralBib}

\begin{thebibliography}{10}
\providecommand{\url}[1]{#1}
\csname url@samestyle\endcsname
\providecommand{\newblock}{\relax}
\providecommand{\bibinfo}[2]{#2}
\providecommand{\BIBentrySTDinterwordspacing}{\spaceskip=0pt\relax}
\providecommand{\BIBentryALTinterwordstretchfactor}{4}
\providecommand{\BIBentryALTinterwordspacing}{\spaceskip=\fontdimen2\font plus
\BIBentryALTinterwordstretchfactor\fontdimen3\font minus
  \fontdimen4\font\relax}
\providecommand{\BIBforeignlanguage}[2]{{%
\expandafter\ifx\csname l@#1\endcsname\relax
\typeout{** WARNING: IEEEtran.bst: No hyphenation pattern has been}%
\typeout{** loaded for the language `#1'. Using the pattern for}%
\typeout{** the default language instead.}%
\else
\language=\csname l@#1\endcsname
\fi
#2}}
\providecommand{\BIBdecl}{\relax}
\BIBdecl

\bibitem{CH}
D.~Callaway and I.~Hiskens, ``Achieving controllability of electric loads,''
  \emph{Proceedings of the IEEE}, vol.~99, no.~1, pp. 184 -- 199, Jan. 2011.

\bibitem{Veloso-RoboSoccer}
M.~Veloso, P.~Stone, K.~Han, and S.~Achim, ``{CMU}nited: A team of robotic
  soccer agents collaborating in an adversarial environment,'' in \emph{{{H}.
  {K}itano, editor, RoboCup-97: The First Robot World Cup Soccer Games and
  Conferences}}.\hskip 1em plus 0.5em minus 0.4em\relax Springer Verlag, 1997,
  pp. 242--256.

\bibitem{Yuta}
S.~Yuta and S.~Premvuti, ``Coordinating autonomous and centralized decision
  making to achieve cooperative behaviors between multiple mobile robots,'' in
  \emph{lEEE/RSJ International Conference on Intelligent Robots and Systems},
  7-10 July 1992, pp. 1566--1574.

\bibitem{Watkins-Q}
C.~Watkins and P.~Dayan, ``${Q}$-learning,'' \emph{Machine Learning}, vol.~8,
  pp. 279--292, 1992.

\bibitem{Tsitsiklis-Q}
J.~Tsitsiklis, ``Asynchronous stochastic approximation and ${Q}$-learning,''
  \emph{Machine Learning}, vol.~16, pp. 185--202, 1994.

\bibitem{Jaakkola-Q}
T.~Jaakkola, M.~Jordan, and S.~Singh, ``On the convergence of stochastic
  iterative dynamic programming algorithms,'' \emph{Neural Computation},
  vol.~6, no.~6, pp. 1185--1201, 1994 1992.

\bibitem{Sutton-Q}
R.~Sutton, A.~Barto, and R.~Williams, ``Reinforcement learning is direct
  adaptive control,'' \emph{IEEE Control Systems Magazine}, pp. 19 -- 22, April
  1992.

\bibitem{Barto-Q}
A.~Barto, S.~Bradtke, and S.~Singh, ``Real-time learning and control using
  asynchronous dynamic programming,'' \emph{Artificial Intelligence}, 1995.

\bibitem{Shoham-Survey}
Y.~Shoham, R.~Powers, and T.~Grenager, ``Multi-agent reinforcement learning: a
  critical survey,'' May 2003, computer Science Dept., Stanford University,
  Stanford, CA. [Online]:
  \url{http://ece.ut.ac.ir/classpages/F85/ControlOfStochasticSystems/res/Multi_Agent_Reinforcement_Learning.pdf}.

\bibitem{Busoniu-Survey}
L.~Busoniu, R.~Babuska, and B.~Schutter, ``A comprehensive survey of multiagent
  reinforcement learning,'' \emph{IEEE Transactions on Systems, Man, and
  Cybernetics - Part C: Applications and Reviews}, vol.~38, no.~2, pp.
  156--172, March 2008.

\bibitem{Littman-SG}
M.~Littman, ``Markov games as a framework for multi-agent reinforcement
  learning,'' in \emph{The 11th International Conference on Machine Learning},
  1994, pp. 157--163.

\bibitem{Littman-SG1}
M.~Littman and C.~Szepesvari, ``A generalized reinforcement learning model:
  convergence and applications,'' in \emph{The 13th International Conference on
  Machine Learning}, 1996, pp. 310--318.

\bibitem{Hu-SG}
J.~Hu and P.~Wellman, ``Multiagent reinforcement learning: theoretical
  framework and an algorithm,'' in \emph{The 15th International Conference on
  Machine Learning}, 1998, pp. 242--250.

\bibitem{Bowling-SG}
M.~Bowling and M.~Veloso, ``Rational and convergent learning in stochastic
  games,'' in \emph{The 17th International Joint Conference on Artificial
  Intelligence}, 2001, pp. 1021--1026.

\bibitem{Claus-SG}
C.~Claus and C.~Boutilier, ``The dynamics of reinforcement learning in
  cooperative multiagent systems,'' in \emph{The 15th International Conference
  on Artificial Intelligence}, 1998, pp. 746--752.

\bibitem{Busoniu-38}
M.~Littman, ``Value function reinforcement learning in {M}arkov games,''
  \emph{J. Cogn. Syst. Res.}, vol.~2, no.~1, pp. 55--66, 2001.

\bibitem{Busoniu-41}
M.~Lauer and M.~Riedmiller, ``An algorithm for distributed reinforcement
  learning in cooperative multi-agent systems,'' in \emph{The 17th
  International Conference on Machine Learning}, Stanford Univ., Stanford, CA,
  Jun. 29 - Jul. 2 2000, pp. 535--542.

\bibitem{Busoniu-45}
C.~Guestrin, M.~Lagoudakis, and R.~Parr, ``Coordinated reinforcement
  learning,'' in \emph{The 19th International Conference on Machine Learning},
  Sydney, Australia, Jul. 8-12 2002, pp. 227--234.

\bibitem{Busoniu-46}
J.~Kok, M.~Spaan, and N.~Vlassis, ``Non-communicative multi-robot coordination
  in dynamic environment,'' \emph{Robot. Auton. Syst.}, vol.~50, no. 2-3, pp.
  99--114, 2005.

\bibitem{Busoniu-48}
J.~Kok and N.~Vlassis, ``Using the max-plus algorithm for multiagent decision
  making in coordination graphs,'' in \emph{Robot Soccer World Cup IX (RoboCup
  2005)}, Osaka, Japan, Jul. 13-19 2005.

\bibitem{Busoniu-78}
D.~Pynadath and M.~Tambe, ``The communicative multiagent team decision problem:
  analyzing teamwork theories and models,'' \emph{J. Artif. Intell. Res.},
  vol.~16, pp. 389--423, 2002.

\bibitem{Busoniu-63}
A.~Kapetanakis and D.~Kudenko, ``Reinforcement learning of coordination in
  cooperative multi-agent systems,'' in \emph{The 18th Nat. Conf. Artif. Intell
  and 14th Conf. Innov. Appl. Artif. Intell.}, Menlo Park, CA, Jul. 28 - Aug. 1
  2002, pp. 326--331.

\bibitem{Veloso-elsevier}
F.~Melo and M.~Veloso, ``Decentralized {MDP}s with sparse interactions,''
  \emph{Artificial Intelligence}, vol. 175, no.~11, pp. 1757--–1789, July 2011.

\bibitem{KarMouraRamanan-Est}
S.~Kar, J.~Moura, and K.~Ramanan, ``Distributed parameter estimation in sensor
  networks: nonlinear observation models and imperfect communication,''
  \emph{IEEE Transactions on Information Theory}, vol.~58, no.~6, June 2012, to
  Appear. Initial Submission: Aug. 2008. [Online]:
  \url{http://arxiv.org/abs/0809.0009}.

\bibitem{JSTSP-Kar-Moura}
S.~Kar and J.~M.~F. Moura, ``Convergence rate analysis of distributed gossip
  (linear parameter) estimation: Fundamental limits and tradeoffs,'' \emph{IEEE
  Journal of Selected Topics in Signal Processing: Signal Processing in
  Gossiping Algorithms Design and Applications}, vol.~5, no.~4, pp. 674--690,
  August 2011.

\bibitem{SICON-Kar-Moura-Poor}
S.~Kar, J.~Moura, and H.~Poor, ``Distributed linear parameter estimation:
  asymptotically efficient adaptive strategies,'' 2011, submitted to the SIAM
  J. Control Optim., Initial Submission: Sept. 2011.
  \url{http://arxiv.org/abs/1109.4960}.

\bibitem{Sayed-LMS}
C.~G. Lopes and A.~H. Sayed, ``Diffusion least-mean squares over adaptive
  networks: Formulation and performance analysis,'' \emph{IEEE Transactions on
  Signal Processing}, vol.~56, no.~7, pp. 3122--3136, July 2008.

\bibitem{Bajovic-detection}
D.~Bajovic, D.~Jakovetic, J.~Moura, J.~Xavier, and B.~Sinopoli, ``Large
  deviations analysis of consensus+innovations detection in random networks,''
  in \emph{The 49th Annual Allerton Conference on Control, Communication, and
  Computing}, Monticello, IL, Sept. 28 - 30 2011, pp. 151--155.

\bibitem{dimakiskarmourarabbatscaglione-11}
A.~G. Dimakis, S.~Kar, J.~M.~F. Moura, M.~G. Rabbat, and A.~Scaglione, ``Gossip
  algorithms for distributed signal processing,'' \emph{Proceedings of the
  IEEE}, vol.~98, no.~11, pp. 1847--1864, Nov 2010.

\bibitem{jadbabailinmorse03}
A.~Jadbabaie, J.~Lin, and A.~S. Morse, ``Coordination of groups of mobile
  autonomous agents using nearest neighbor rules,'' \emph{IEEE Transactions on
  Automatic Control}, vol.~48, no.~6, pp. 988--1001, Jun. 2003.

\bibitem{SensNets:Olfati04}
R.~Olfati-Saber and R.~M. Murray, ``Consensus problems in networks of agents
  with switching topology and time-delays,'' \emph{IEEE Trans. Automat.
  Contr.}, vol.~49, no.~9, pp. 1520--1533, Sept. 2004.

\bibitem{olfatisaberfaxmurray07}
R.~Olfati-Saber, J.~A. Fax, and R.~M. Murray, ``Consensus and cooperation in
  networked multi-agent systems,'' \emph{Proceedings of the IEEE}, vol.~95,
  no.~1, pp. 215--233, January 2007.

\bibitem{karmoura-randomtopologynoise}
S.~Kar and J.~M.~F. Moura, ``Distributed consensus algorithms in sensor
  networks with imperfect communication: Link failures and channel noise,''
  \emph{IEEE Transactions on Signal Processing}, vol.~57, no.~1, pp. 355--369,
  January 2009.

\bibitem{Nedic}
A.~Nedic, A.~Olshevsky, A.~Ozdaglar, and J.~N. Tsitsiklis, ``On distributed
  averaging algorithms and quantization effects,'' \emph{IEEE Transactions on
  Automatic Control}, no.~11, pp. 2506--2517, Nov. 2009.

\bibitem{tsitsiklisbertsekasathans86}
J.~N. Tsitsiklis, D.~P. Bertsekas, and M.~Athans, ``Distributed asynchronous
  deterministic and stochastic gradient optimization algorithms,'' \emph{IEEE
  Transactions on Automatic Control}, vol.~31, no.~9, pp. 803--812, September
  1986.

\bibitem{jakoveticxaviermoura-11}
D.~Jakovetic, J.~Xavier, and J.~Moura, ``Cooperative convex optimization in
  networked systems: Augmented {L}agrangian algorithms with directed gossip
  communication,'' \emph{IEEE Transactions on Signal Processing}, vol.~59,
  no.~8, pp. 3889--3902, August 2011.

\bibitem{Giannakis-opt}
G.~Mateos, J.~Bazerque, and G.~Giannakis, ``Distributed sparse linear
  regression,'' \emph{IEEE Transactions on Signal Processing}, vol.~58, no.~11,
  pp. 5262--5276, Nov. 2010.

\bibitem{Nedic-opt}
A.~Nedic and A.~Ozdaglar, ``Distributed subgradient methods for multi-agent
  optimization,'' \emph{IEEE Transactions on Automatic Control}, vol.~54,
  no.~1, p. 48 – 61, Jan. 2009.

\bibitem{rabbatnowakbucklew05}
M.~G. Rabbat, R.~D. Nowak, and J.~A. Bucklew, ``Generalized consensus
  algorithms in networked systems with erasure links,'' in \emph{6th IEEE
  Workshop on Signal Processing Advances in Wireless Comunnications (SPAWC)},
  New York, NY, 2005, pp. 1088--1092.

\bibitem{Ram-Nedich-Siam}
S.~S. Ram, A.~Nedic, and V.~V. Veeravalli, ``Incremental stochastic subgradient
  algorithms for convex optimization,'' \emph{SIAM Journal on Optimization},
  vol.~20, no.~2, pp. 691--717, June 2009.

\bibitem{FanChung}
F.~R.~K. Chung, \emph{Spectral {G}raph {T}heory}.\hskip 1em plus 0.5em minus
  0.4em\relax Providence, RI : American Mathematical Society, 1997.

\bibitem{Mohar}
B.~Mohar, ``The {L}aplacian spectrum of graphs,'' in \emph{Graph Theory,
  Combinatorics, and Applications}, Y.~Alavi, G.~Chartrand, O.~R. Oellermann,
  and A.~J. Schwenk, Eds.\hskip 1em plus 0.5em minus 0.4em\relax New York: J.
  Wiley \& Sons, 1991, vol.~2, pp. 871--898.

\bibitem{Bertsekas-DP}
D.~Bertsekas, \emph{Dynamic Programming and Stochastic Control}.\hskip 1em plus
  0.5em minus 0.4em\relax New York, NY: {A}cademic {P}ress, {I}nc., 1976.

\bibitem{Boyd-Gossip}
S.~Boyd, A.~Ghosh, B.~Prabhakar, and D.~Shah, ``Randomized gossip algorithms,''
  \emph{IEEE/ACM Trans. Netw.}, vol.~14, no.~SI, pp. 2508--2530, 2006.

\bibitem{tm04}
S.~Tatikonda and S.~Mitter, ``Control under communication constraints,''
  \emph{IEEE Transactions on Automatic Control}, vol.~49, no.~7, pp. 1056 --
  1068, July 2004.

\bibitem{ms03}
A.~Matveev and A.~Savkin, ``The problem of state estimation via asynchronous
  communication channels with irregular transmission times,'' \emph{IEEE
  Transactions on Automatic Control}, vol.~48, no.~4, pp. 670--676, April 2006.

\bibitem{Li-Baillieul}
K.~Li and J.~Baillieul, ``Robust and efficient quantization and coding for
  control of multidimensional linear systems under data rate constraints,''
  \emph{International Journal of Robust and Nonlinear Control Special Issue:
  Communicating-Agent Networks}, vol.~17, no. 10-11, pp. 898--920, July 2007.

\bibitem{karmoura-quantized}
\BIBentryALTinterwordspacing
S.~Kar and J.~M.~F. Moura, ``Distributed consensus algorithms in sensor
  networks: Quantized data and random link failures,'' \emph{IEEE Transactions
  on Signal Processing}, vol.~58, no.~3, pp. 1383--1400, March 2010, initial
  post: Dec. 2007. [Online]. Available: \url{http://arxiv.org/abs/0712.1609}
\BIBentrySTDinterwordspacing

\bibitem{Jacod-Shiryaev}
J.~Jacod and A.~Shiryaev, \emph{Limit Theorems for Stochastic Processes}.\hskip
  1em plus 0.5em minus 0.4em\relax Berlin Heidelberg: Springer-Verlag, 1987.

\bibitem{Szepesvari-Q}
C.~Szepesvari, ``The asymptotic convergence-rate of ${Q}$-learning,'' in
  \emph{Advances in Neural Information Processing Systems}, M.~Jordan,
  M.~Kearns, and S.~Solla, Eds., 1998, vol.~10, p. 1064 – 1070.

\bibitem{Dar-Q}
E.~Even-Dar and Y.~Mansour, ``Learning rates for ${Q}$-learning,''
  \emph{Journal of Machine Learning Research}, vol.~5, pp. 1--25, 2003.

\end{thebibliography}

\end{document}